\documentclass[lettersize,journal]{IEEEtran}
\usepackage{amsmath,amsfonts}
\usepackage{algorithm}

\usepackage{array}
\usepackage{textcomp}
\usepackage{stfloats}
\usepackage{verbatim}
\usepackage{graphicx}
\usepackage{cite}
\usepackage[hyphens]{url}

\usepackage{amsmath,amssymb,amsfonts}
\usepackage{graphicx}
\usepackage{textcomp}
\usepackage{xcolor}
\usepackage{amsthm}
\usepackage{amssymb}

\usepackage{setspace}
\usepackage{psfrag}
\usepackage{adjustbox}
\usepackage{multirow}
\theoremstyle{definition}
\newtheorem{theorem}{Theorem}
\newtheorem{lemma}[theorem]{Lemma}
\usepackage{comment}
\usepackage{subcaption}
\usepackage{float}
\usepackage{booktabs}

\newcommand{\bd}[1]{\mathbf{#1}}
\newcommand{\set}[1]{\mathcal{#1}}

\usepackage[english]{babel}
\usepackage{overpic}
\usepackage{algorithm}% http://ctan.org/pkg/algorithms
\usepackage{algpseudocode}% http://ctan.org/pkg/algorithmicx

\algrenewcommand\algorithmicrequire{\textbf{Input:}}
\algrenewcommand\algorithmicensure{\textbf{Output:}}
\algrenewcommand\algorithmicforall{\textbf{for}}
\DeclareMathOperator*{\argmax}{arg\,max}
\usepackage{xurl}  

\usepackage[breaklinks=true]{hyperref}
\usepackage{breakurl}
\begin{document}

\title{Adaptive Multi-task Learning for\\ Probabilistic Load Forecasting}

\author{Onintze Zaballa, Ver\'onica \'Alvarez, and Santiago Mazuelas,~\IEEEmembership{Senior Member,~IEEE}

\thanks{Manuscript received XXXX, YYYY; revised XXXX, YYYY.

Funding in direct support of this work has been provided by projects PID2022-137063NB-I00 and CEX2021-001142-S funded by MCIN/AEI/10.13039/501100011033 and the European Union “NextGenerationEU”/PRTR, and programes ELKARTEK and BERC-2022-2025 funded by the Basque Government. Verónica Álvarez holds a postdoctoral grant from the
Basque Government.

O. Zaballa was with the BCAM-Basque Center for Applied Mathematics,
Bilbao 48009, Spain. She is now with Zenit Solar Tech, Donostia-San Sebastián 20012, Spain  (e-mail: o.zaballa@zenitsolar.eu).

V. \'Alvarez is with the Massachusetts Institute of Technology (MIT), Cambridge 02139, USA, and BCAM-Basque Center for Applied Mathematics, Bilbao 48009, Spain (e-mail: vealvar@mit.edu).

S. Mazuelas is with the BCAM-Basque Center for Applied Mathematics, and IKERBASQUE-Basque Foundation for Science, Bilbao 48009, Spain (e-mail: smazuelas@bcamath.org).
}}
\markboth{IEEE Transactions on Power Systems}%
{Shell \MakeLowercase{\textit{et al.}}: A Sample Article Using IEEEtran.cls for IEEE Journals}


\maketitle

\begin{abstract}
Simultaneous  load forecasting across multiple entities (e.g., regions, buildings) is crucial for the efficient, reliable, and cost-effective operation of power systems. Accurate load forecasting is a challenging problem due to the inherent uncertainties in load demand, dynamic changes in consumption patterns, and correlations among entities. Multi-task learning has emerged as a powerful machine learning approach that enables the simultaneous learning across multiple related problems. However, its application to load forecasting remains under-explored and is limited to offline learning methods, which cannot capture changes in consumption patterns.  This paper presents an adaptive multi-task learning method for probabilistic load forecasting. The proposed method can dynamically adapt to changes in consumption patterns and correlations among entities. In addition, the techniques presented provide reliable probabilistic predictions for loads of multiple entities and assess load uncertainties. Specifically, the method is based on vector-valued hidden Markov models and uses a recursive process to update the model parameters and provide predictions with the most recent parameters. The performance of the proposed method is evaluated using datasets that contain the load demand of multiple entities and exhibit diverse and dynamic consumption patterns. The experimental results show that the presented techniques outperform existing methods both in terms of forecasting performance and uncertainty assessment.
\end{abstract}

\begin{IEEEkeywords}
Multi-task learning, Probabilistic load forecasting, Online learning.
\end{IEEEkeywords}

\section{Introduction}

\IEEEPARstart{L}{oad} forecasting is essential for ensuring efficient energy generation, sustaining supply-demand balance, and achieving cost-effective power operations \cite{lee2023multi, guo2021forecast}. With the advent of smart grids, accurate load forecasting has become increasingly important also to maintain network stability and to enhance bidding strategies in power markets \cite{wang2018review, wang2023multitask}. In power systems, accurately predicting the energy associated with multiple entities (e.g., regions, buildings, or neighborhoods) is crucial for optimizing energy distribution, managing demand, and integrating renewable resources \cite{melo2012multi, carreno2010cellular, savino2025scalable, kaspar2022critical}. For instance, load forecasting at the neighborhood level in smart grids helps balance supply and demand more efficiently because it enables localized control strategies, reduces the need for over-provisioning, and supports real-time decision-making for distributed energy resources. However, load forecasting remains challenging due to the complexity of the electric grid, the inherent uncertainties in load consumption, the dynamic changes in consumption patterns, and the correlations among entities \cite{lin2024electric, obst2021adaptive, lin2021spatial} (see Fig.~\ref{fig:graphicalmodel}).

\begin{figure}
    \centering
        \psfrag{a}[][l][0.8]{\hspace{0.3cm} $p(\bd{s}_{t}|\bd{s}_{t-1})$}
        \psfrag{b}[][l][0.8]{ $p(\bd{r}_{t-1}|\bd{s}_{t-1})$ \hspace{0.7cm}}
        \psfrag{c}[][l][0.8]{\hspace{0.2cm}  $p(\bd{r}_{t}|\bd{s}_{t})$ \hspace{0.3cm}}
        \psfrag{d}[][l][0.8]{\hspace{0.1cm}  $\bd{s}_{t-1}$}
        \psfrag{e}[][l][0.8]{\hspace{0.1cm}  $\bd{r}_{t-1}$}
        \psfrag{f}[][l][0.8]{ $\bd{s}_{t}$}
        \psfrag{g}[][l][0.8]{ $\bd{r}_{t}$}
        \includegraphics[width=\columnwidth]{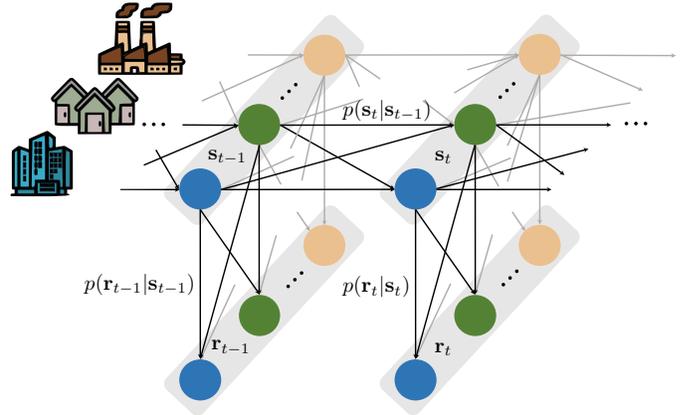}
    \caption{The proposed multi-task method leverages information from multiple entities to learn the consumption patterns over time. In particular, a vector-valued HMM enables to capture the relationship between multiple consecutive loads, $p(\bd{s}_t|\bd{s}_{t-1})$, and between multiple loads and observations, $p(\bd{r}_t|\bd{s}_t)$.}
    \label{fig:graphicalmodel}
\end{figure}

Forecasting methods exploit consumption patterns characterized by various factors including past loads, day times, holidays, and weather conditions \cite{ranaweera1997economic, taylor2002neural}. However, these patterns are inherently dynamic and uncertain, as they evolve with changes in consumer behavior and are affected by unforeseen events \cite{li2023residential, kong2017short}. Online learning techniques are necessary to effectively adapt to dynamic changes in consumption patterns and continuously update models as new data arrive~\mbox{\cite{von2020online, bahrami2017online}}. In addition, probabilistic forecasting is essential for quantifying uncertainties in load demand and supporting better informed decisions in power operations~\mbox{\cite{draz2024probabilistic, sun2019using}}. Probabilistic forecasting becomes even more crucial  for simultaneous load forecasting across multiple entities since probabilistic forecasts can enable robust decision-making that optimally allocates resources and anticipates cascading effects across the entire power system \cite{samadi2013tackling, hernandez2014survey, yang2019bayesian, wang2022novel}. For instance, probabilistic load forecasting at the neighborhood level in smart grids allows operators to allocate resources accounting for the estimated uncertainty. A significant uncertainty in a neighborhood may trigger additional reserves, while a reduced uncertainty can enable to more efficiently allocate resources. Moreover, if the estimated correlation between two neighborhoods is negative, a higher than expected demand in one may be offset by lower demand in the other, reducing overall system risk.  However, existing learning methods that
are both probabilistic and online~\cite{alvarez2021probabilistic} are designed for single-task learning that separately provide load forecasts for each individual entity. 

In the machine learning literature, multi-task learning has emerged as a powerful approach that enables simultaneous learning across multiple related problems \cite{caruana1997multitask, deng2022multi}. Multi-task learning can improve overall performance by capturing both task-specific patterns and correlations among assorted tasks. These techniques usually learn dependencies among assorted tasks by using training data from all tasks \mbox{\cite{zhang2021survey, tripuraneni2020theory}}. 
Multi-task learning can be exploited to enhance the load predictions for each entity by simultaneously learning from multiple entities and leveraging the relationships among their consumption patterns. 

Multi-task learning remains under-explored in
the context of load forecasting. Only a few load forecasting
techniques have applied multi-task learning to forecast loads across multiple entities \cite{Zhang2014, gilanifar2019multitask, fiot2016electricity}. For instance, multi-task learning methods based on Gaussian processes can predict the load demand across  multiple entities by using correlations from historical load demand data~\cite{Zhang2014}. Methods based on Bayesian spatiotemporal Gaussian processes can predict load demand across  multiple entities by using correlations obtained from dynamic environmental and traffic conditions~\cite{gilanifar2019multitask}. In addition, multi-task learning methods based on kernels can obtain mid-term predictions by exploiting correlations in seasonal patterns~\cite{fiot2016electricity}. However, existing load forecasting techniques that use multi-task learning are based on offline learning and cannot dynamically adapt to changes in consumption patterns. In addition, current methods are not able to accurately assess load uncertainty in multi-task scenarios, and often result in high computational complexities.

\begin{figure*}
    \centering
        \psfrag{A}[][t][0.8]{New South Wales}
        \psfrag{C}[][t][0.8]{Tasmania}
        \psfrag{B}[][t][0.8]{Queensland}
        \psfrag{D}[][t][0.8]{South Australia}
        \psfrag{Day}[][t][0.7]{Day of year}
        \psfrag{Hour}[][b][0.7]{Hour of day}
        \psfrag{0}[][l][0.5]{}
        \psfrag{6}[][][0.45]{6}
        \psfrag{12}[][][0.45]{12}
        \psfrag{18}[][][0.45]{18}
        \psfrag{50}[][][0.45]{50}
        \psfrag{100}[][][0.45]{100}
        \psfrag{150}[][][0.45]{150}
        \psfrag{200}[][][0.45]{200}
        \psfrag{250}[][][0.45]{250}
        \psfrag{300}[][][0.45]{300}     
        \psfrag{350}[][][0.45]{350}
        \psfrag{2}[][][0.45]{2}
        \psfrag{2.5}[][c][0.45]{2.5}
        \psfrag{3}[][][0.45]{3}
        \psfrag{3.5}[][c][0.45]{3.5}
        \psfrag{4}[][][0.45]{4}
        \psfrag{4.5}[][c][0.45]{4.5}
        \psfrag{5}[][][0.45]{5}
        \psfrag{0.6}[][][0.5]{\hspace{0.1cm} 0.6}
        \psfrag{8}[][][0.45]{8}
        \psfrag{10}[][][0.45]{10}
        \psfrag{14}[][][0.45]{14} 
        \psfrag{16}[][][0.45]{16}  
        \psfrag{20}[][][0.45]{20}  
        \psfrag{22}[][][0.45]{22} 
        \psfrag{1.2}[][][0.45]{1.2}
        \psfrag{1.6}[][][0.45]{1.6} 
        \psfrag{2.4}[][][0.45]{2.4}  
        \psfrag{2.8}[][][0.45]{2.8}  
        \psfrag{3.2}[][][0.45]{3.2} 
        \psfrag{1.8}[][][0.45]{1.8}  
        \psfrag{1.4}[][][0.45]{1.4}  
        \psfrag{1}[][][0.45]{1}  
        \psfrag{0.8}[][][0.45]{0.8} 
        \psfrag{0.4}[][][0.45]{0.4} 
        \psfrag{5}[][][0.45]{5}  
        \psfrag{7}[][][0.45]{7} 
        \psfrag{9}[][][0.45]{9}  
        \psfrag{11}[][][0.45]{11}  
        \begin{overpic}[width=1\textwidth]{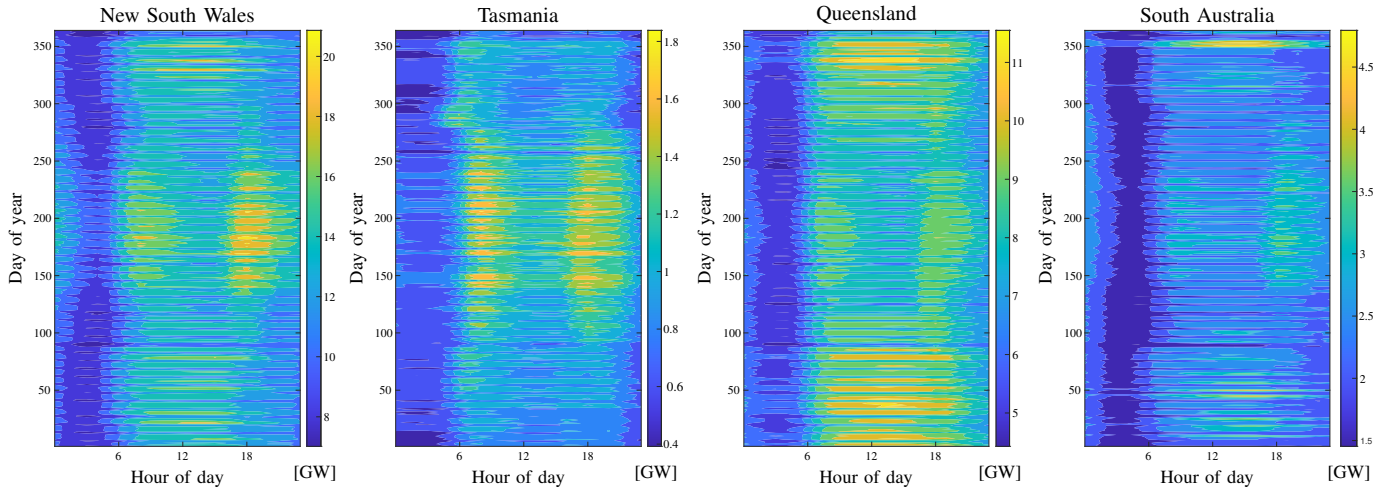}
        \put(20.5, 1.5){\scriptsize [GW]}
        \put(45.5, 1.5){\scriptsize [GW]}
        \put(70.7, 1.5){\scriptsize [GW]}
        \put(96.3, 1.5){\scriptsize [GW]}
    \end{overpic}
    \caption{Heat maps corresponding to hourly and daily loads of four entities in Australia. The consumptions in New South Wales  and Tasmania exhibit similar patterns, the consumptions in Queensland show moderate differences from New South Wales, and the consumptions in South Australia are substantially different. 
     The figure also shows how such similarities can  change often over time.}
    \label{fig:patterns}

\end{figure*}

In this paper, we present adaptive multi-task learning methods for probabilistic load forecasting (\mbox{Multi-APLF}). The proposed method can dynamically adapt to changes in consumption patterns and correlations among entities. In addition, the techniques presented provide reliable probabilistic predictions for loads of multiple entities. Specifically, the main contributions of the paper are as follows.
\begin{itemize}\itemsep0em 
    \item We develop online learning methods for multi-task load forecasting based on vector-valued hidden Markov models (HMMs). In particular, the presented methods can effectively account for the relationship among the consumption patterns of multiple entities (see Fig.~\ref{fig:graphicalmodel}).
    \item We develop sequential prediction techniques for \mbox{Multi-APLF} that obtain probabilistic forecasts for multiple energy loads  using the most recently learned HMM parameters. These techniques recursively generate predictions, estimating both the uncertainty of each entity and the correlations between different entities.
    \item We provide a detailed description of the implementation of the online learning and probabilistic forecasting steps, along with their computational complexity analysis. Our approach is significantly more efficient than existing multi-task techniques, making it suitable for large-scale applications involving multiple entities.
    \item The experimental results show that the presented methods outperform existing methods both in terms of forecasting performance and uncertainty assessment.
\end{itemize}

The multi-task method proposed in this paper generalizes the single-task techniques introduced in \cite{alvarez2021probabilistic}. Although HMMs are well established in other domains, their application to load forecasting has been limited  and is restricted to the single-task settings \cite{alvarez2021probabilistic, hermias2017short, wang2020short}.  This paper extends the HMM framework introduced in \cite{alvarez2021probabilistic} to simultaneously model and forecast multiple entities. This requires learning a richer probabilistic structure that captures both the individual uncertainty of each entity and the correlations among them. As a result, the number of HMM parameters is substantially larger than in the single-task case.  The generalization to multi-task learning leverages the relationships in consumption patterns among multiple entities to enhance the accuracy of load forecasting for each individual entity. Figure~\ref{fig:patterns} illustrates the relationship and varying similarity among consumptions per hour of day and day of the week of four different entities  corresponding to Australian regions. The figure also shows that consumption patterns   change often over time, not only for different seasons.

Preliminary results for the multi-task generalization have been presented as a conference paper \cite{zaballa2024multitask}. 
The current work  introduces four major advancements beyond that preliminary work: (i) the development of theoretical guarantees for the proposed method with formal analysis and proofs; (ii) a full implementation and detailed algorithmic description; (iii) a computational complexity analysis of the proposed methods with existing approaches; and (iv) an extensive experimental evaluation of predictive accuracy and probabilistic performance against additional existing methods.

The rest of this paper is organized as follows. Section~\ref{sec:loadforecasting} describes the load forecasting problem. In Section \ref{sec:multitask}, we present the learning and prediction steps of the proposed method, as well as its implementation. Section \ref{sec:experiments} evaluates the performance of the proposed method in comparison with existing techniques. Finally, Section \ref{sec:conclusion} draws the conclusions.

\paragraph*{Notations} bold lowercase letters represent vectors; bold capital letters represent matrices; $\set{N}\left( \bd{x}; \boldsymbol{\mu}, \boldsymbol{\Sigma} \right)$ denotes the multivariate Gaussian density function of the variable $\bd{x}$ with mean $\boldsymbol{\mu}$ and covariance matrix $\boldsymbol{\Sigma}$; $p\left(\bd{x}|\bd{y} \right)$ denotes the probability of variable $\bd{x}$ given variable $\bd{y}$; $\mathbb{I}(\cdot)$ indicates denotes the indicator function; $\propto$ indicates that two expressions are proportional;  $\bd{I}_d$ denotes the $d \times d$ identity matrix; $\bd{0}$ denotes the matrix with all components equal to 0; $\bd{r}_{i:j}$ denotes the array of vectors $[\bd{r}_{i}^\top, \bd{r}_{i+1}^\top,..., \bd{r}_{j}^\top]^\top$;  and $[\, \cdot \,]^{\top}$ denotes the transpose  of its argument.% % \vspace{-0.15cm}

\section{Load forecasting}
\label{sec:loadforecasting}

This section first formulates the problem of load forecasting, and
then briefly describes the performance metrics used.

\subsection{Problem formulation}

 In power systems, load forecasting methods aim to estimate future demand across multiple entities using historical consumption data together with load-related observations.  Let \mbox{$\bd{s}= [s_{1},s_{2},...,s_{K}]^\top \in \mathbb{R}^{K}$} and  \mbox{$\widehat{\bd{s}}=[ \widehat{s}_{1}, \widehat{s}_{2},...,\widehat{s}_{K}]^\top \in \mathbb{R}^{K}$} denote the vector of actual and forecast loads, respectively, for $K$ entities.  
 For instance, each $s_i$ with \mbox{$i\in \{1,2,...,K\}$} may represent the load demand  in a specific neighborhood, region, or building within a power distribution network, and an accurate $\widehat{s}_i$ for such neighborhood is critical for scheduling generation, sharing energy with other regions, reducing operational costs, and ensuring system stability.

Load-related observations include factors that influence future load demands, such as weather forecasts or renewable generation forecasts. 
Let $\bd{r}_{i}\in \mathbb{R}^{R}$ denote the observation vector for each entity \mbox{$i\in \{1,2,...,K\}$}, and \mbox{$\bd{r}=[ \bd{r}_{1}^\top, \bd{r}_{2}^\top,...,\bd{r}_{K}^\top]^\top\in \mathbb{R}^{KR}$} denote the array composed of concatenating the observation vectors of the $K$ entities of the power system. 
In addition, let $c(t) \in \{1,2,...,C\}$ be a calendar variable that describes time factors that may affect load demand at each time $t$, such as the hour of the day, day of the week, or holiday period. The usage of such calendar information is highly valuable for forecasting the loads of multiple entities, for instance  residential areas typically exhibit lower consumption during working hours and higher consumption in the evening, while commercial zones may show the opposite pattern.

Load forecasting techniques obtain a prediction function that relates instance vectors $\bd{x}$ to target vectors $\bd{y}$. In multi-task techniques, instance vectors $\bd{x}$ (predictors) are composed of past loads of $K$ entities and observations related to future loads. Target vectors $\bd{y}$ (responses) are composed of future loads of $K$ entities. Then, for a prediction horizon $L$ (e.g., 24-48 hours, 30 minutes) and prediction times \mbox{$t+1,$} \mbox{$t+2,$} ..., \mbox{$t+L$}, the instance vector $\bd{x}$ is composed of load vectors $\bd{s}_{t-m+1}^\top, \bd{s}_{t-m+2}^\top,..., \bd{s}_t^\top$ corresponding to $m$ previous times, and of observation vectors $\bd{r}_{t+1}^\top,\bd{r}_{t+2}^\top, ..., \bd{r}_{t+L}^\top$ so that \mbox{$\bd{x}\in \mathbb{R}^{K (m+LR)}$}. In addition, the target vector is defined as \mbox{$\bd{y} = [\bd{s}_{t+1}^\top,\bd{s}_{t+2}^\top,...,\bd{s}_{t+L}^\top]^\top \in \mathbb{R}^{K L}$}, and the vector of load forecasts is given by \mbox{$\widehat{\bd{y}} = [\widehat{\bd{s}}_{t+1}^\top,\widehat{\bd{s}}_{t+2}^\top,...,\widehat{\bd{s}}_{t+L}^\top]^\top \in \mathbb{R}^{KL}$}. 

Multi-task load forecasting is a high-dimensional learning problem because it involves the simultaneous predictions of multiple consumptions using models characterized by a large number of parameters. For instance, methods based on multiple linear regression (MLR) obtain load forecasts as \mbox{$\widehat{\bd{y}}=\bd{W}^\top \bd{x}$}. The weight matrix is determined as \mbox{$\bd{W}= (\bd{X}^\top \bd{X})^{-1} \bd{X}^\top \bd{Y} \in \mathbb{R}^{K(m + LR) \times KL}$}, where the parameters are estimated using $N$ training samples $\{(\bd{x}_i, \bd{y}_i)\}_{i=1}^N$~\cite{bishop2006pattern}. These samples form the matrices $\bd{X}\in\mathbb{R}^{N \times K(m+LR)}$ and $\bd{Y}\in\mathbb{R}^{N \times KL}$. MLR models require a large number of parameters, for instance, considering $K=8$ entities, $m=20$ previous loads and $R=3$ observations, the matrix $\bd{W}$ is composed of $141,312$ parameters for a prediction horizon of $L=24$.

\subsection{Performance metrics}

Performance of forecasting algorithms is evaluated in terms of accuracy using the absolute value of prediction errors: 
\begin{equation}
\label{predictionerror}
e = \left|s - \widehat{s}\right|
\end{equation}
while probabilistic performance can be evaluated using metrics such as continuous ranked probability score (CRPS) \cite{gneiting2007strictly}, calibration \cite{gneiting2007probabilistic}, and pinball loss \cite{hong2016probabilistic}. Overall prediction errors are commonly quantified using root mean square error (RMSE) given by
\begin{equation*}
\text{RMSE} = \sqrt{\mathbb{E}\left\{\left|s - \widehat{s}\right|^2\right\}}
\end{equation*}
and mean average percentage error (MAPE) given by
\begin{equation*}
\text{MAPE} = 100 \cdot \mathbb{E}\left\{\frac{\left|s - \widehat{s}\right|}{s}\right\}.
\end{equation*}
The CRPS is given by
\begin{equation}
    \text{CRPS}(F,s) = \int_{-\infty}^\infty  (F(y) - \mathbb{I}\{y \geq s\})^2 dy
\end{equation}
which compares the predictive CDF $F$ in relation to the actual observation $s$. In particular, we compute the CRPS value as shown in \cite{gneiting2007strictly}. In addition, we use the calibration of the $q$-th quantile forecast $\widehat{s}_{(q)}$ given by 
\begin{equation}
    C(q) = \mathbb{E} \{ \mathbb{I}_{s < \widehat{s}_{(q)} } \}
\end{equation}
to assess the probability that the observed load is smaller than the forecasted quantile $\widehat{s}_{(q)}$ \cite{gneiting2007probabilistic}. The overall calibration error~(CE) is quantified by the average over all quantiles of the error $|C(q)-q|$. Finally, the pinball loss for the $q$-th quantile forecast $\widehat{s}_{(q)}$ is defined as
% % \vspace{-0.2cm}
\begin{equation}
\text{pinball}(\widehat{s}, s, q) =
\begin{cases}
q  (s - \widehat{s}) & \text{if } s \geq \widehat{s}_{(q)} \\
(1 - q)  (\widehat{s} - s) & \text{if } s < \widehat{s}_{(q)}
\end{cases}
\end{equation}
which measures the accuracy of quantile forecasts, and the overall pinball loss is quantified by the average over all quantiles \cite{hong2016probabilistic}. A lower value of these scores indicates a better prediction interval.

\section{Multi-task probabilistic load forecasting based on online learning}
\label{sec:multitask}

This section describes the multi-task method for load forecasting based on probabilistic and online learning techniques. The proposed \mbox{Multi-APLF} method leverages information from multiple entities to learn the consumption patterns over time. In particular, a vector-valued HMM enables to capture the relationship between multiple consecutive loads, and between multiple loads and observations.

The vector-valued HMM is characterized by the conditional distributions $p(\bd{s}_t|\bd{s}_{t-1})$, which represent the relationship between multiple consecutive loads, and $p(\bd{r}_t|\bd{s}_t)$, which represents the relationship between the loads and the observation vectors of the entities, with $\bd{s}_t=[s_{t,1}, s_{t,2},..., s_{t,K}]^\top \in \mathbb{R}^{K}$ and  $\bd{r}_t=[ \bd{r}_{t,1}^\top, \bd{r}_{t,2}^\top,...,\bd{r}_{t,K}^\top]^\top \in \mathbb{R}^{KR}$ (see Fig.~\ref{fig:graphicalmodel}). These conditional probabilities are modeled using multivariate Gaussian distributions with mean~$\bd{M}\bd{u}$ and covariance matrix~$\bd{\Sigma}$, where $\bd{u}$ is a known feature vector. The parameters described by matrices $\bd{M}$ and $\bd{\Sigma}$ vary for each calendar type \mbox{$c(t) \in \{1,2,...,C\}$} and change over time. The mean matrix~$\bd{M}$ and the covariance matrix~$\bd{\Sigma}$ provide information on how one entity can influence and be influenced by another entity at each time. Specifically, matrix~$\bd{\Sigma}$ describes the correlations between pairs of entities.

For each calendar type $c=c(t)$, let $\bd{M}_{s,c} \bd{u}_s$ and $\bd{\Sigma}_{s,c}$ denote the mean and covariance matrix that characterize the conditional distribution of loads at time $t$ given the loads at time $t-1$, that is, % \vspace{-0.2cm}
\begin{equation}
\label{eq:model1_MT}
  p(\bd{s}_{t}| \bd{s}_{t-1}) = \set{N}(\bd{s}_t;   \bd{M}_{s,c} \bd{u}_{s}, \bd{\Sigma}_{s,c})
  \end{equation}
where $\bd{M}_{s,c}~\in~\mathbb{R}^{K \times (K+1)}$, $\bd{\Sigma}_{s,c} \in \mathbb{R}^{K \times K}$, and $\bd{u}_{s}~=~[1, \bd{s}_{t-1}^\top]^\top \in \mathbb{R}^{K+1}$.

For each calendar type $c=c(t)$, let  $\bd{M}_{r,c} \bd{u}_r$ and $\bd{\Sigma}_{r,c}$ denote the mean and covariance matrix that characterize the conditional distribution of loads at time $t$ and the observation vectors. Specifically, assuming that there is no prior knowledge available for the loads, we have that
\begin{equation}
\label{eq:model2_MT}
    p(\bd{r}_{t}| \bd{s}_{t}) \propto p(\bd{s}_{t}| \bd{r}_{t}) = \set{N}(\bd{s}_t;  \bd{M}_{r,c} \bd{u}_r, \bd{\Sigma}_{r,c})
\end{equation}
where $\bd{M}_{r,c}\in \mathbb{R}^{K \times KR}$,  $\bd{\Sigma}_{r,c} \in \mathbb{R}^{K \times K}$, and $\bd{u}_r~=~u_{r}(\bd{r}_t)~\in \mathbb{R}^{KR}$ is the function that represents the observations. For instance, such a function could simply be $u_r(\bd{r})=\bd{r}$, or a more complex and possibly non-linear function of $\bd{r}$.

% \vspace{-0.3cm}
% 2. Learning
\subsection{Learning}
The vector-valued HMM for each time $t$ is defined by parameters 
\begin{equation}
    \Theta = \{ \bd{M}_{s,c}, \boldsymbol{\Sigma}_{s,c}, \bd{M}_{r,c}, \boldsymbol{\Sigma}_{r,c} : c = 1,2,...,C \}
    \label{eq:theta_parameters}
\end{equation}
where $\bd{M}_{s,c}, \boldsymbol{\Sigma}_{s,c}$ characterize the conditional distribution $p(\bd{s}_t|\bd{s}_{t-1})$ and $\bd{M}_{r,c}, \boldsymbol{\Sigma}_{r,c}$ characterize the conditional distribution $p(\bd{r}_t|\bd{s}_{t})$. Such parameters are updated every time new loads and observations are obtained for each calendar type $c$.

The parameters in $\Theta$ for each calendar type and conditional distribution can be obtained by maximizing the weighted log-likelihood of all the loads observed at times with the same calendar type. Specifically, for each calendar type \mbox{$c \in \{1, 2, ..., C\}$}, let $\bd{s}_{t_1}, \bd{s}_{t_2}, ..., \bd{s}_{t_i} \in \mathbb{R}^{K}$ be the vectors of loads of $K$ entities observed at times with such calendar type $c$, i.e., \mbox{$c = c(t_1) = c(t_2) = ... = c(t_i)$}. In addition, let $\bd{u}_{t_1}, \bd{u}_{t_2} , ..., \bd{u}_{t_i} \in \mathbb{R}^d$ be the corresponding feature vectors for parameters $\bd{M}_{s,c}$, $\boldsymbol{\Sigma}_{s,c}$ as given in \eqref{eq:model1_MT}, or for parameters \mbox{$\bd{M}_{r,c}$, $\boldsymbol{\Sigma}_{r,c}$} as given in \eqref{eq:model2_MT}. The estimates of these parameters are obtained by maximizing the exponentially weighted log-likelihood of the loads observed up to time $t_i$, given by
\begin{equation}
\bd{M}_i,\boldsymbol{\Sigma}_i\in \argmax_{\bd{M},\boldsymbol{\Sigma}} \sum_{j=1}^i \lambda^{i-j} \log \set{N}(\bd{s}_{t_j};  \bd{M} \bd{u}_{t_j}, \bd{\Sigma})
    \label{eq:weighted_loglikelihood}
\end{equation}
where $\lambda \in (0, 1)$ is the forgetting factor parameter that allows to increase the influence of the most recent data. The following theorem shows that the maximization of the weighted log-likelihood in \eqref{eq:weighted_loglikelihood} can be efficiently solved through a recursive process.

\begin{theorem}
\label{theorem1}
The matrices $\bd{M}_i$ and $\boldsymbol{\Sigma}_i$ that maximize the weighted log-likelihood given in \eqref{eq:weighted_loglikelihood} satisfy the following recursions
\begin{align}
    \bd{M}_i = &  \bd{M}_{i-1} + \frac{\bd{e}_i  \bd{u}_{t_i}^\top \bd{P}_{i-1}}{\lambda + \bd{u}_{t_i}^\top \bd{P}_{i-1}\bd{u}_{t_i}} \label{eq:mean_MT}\\
    \bd{\Sigma}_i  = & \bd{\Sigma}_{i-1} - \frac{1}{\bd{\gamma}_i} \Big( \bd{\Sigma}_{i-1}  -  \frac{\lambda^2 \bd{e}_{i} \bd{e}_{i}^\top }{(\lambda+\bd{u}_{t_i}^\top \bd{P}_{i-1}\bd{u}_{t_i})^2} \Big)   \label{eq:sigma_MT}
\end{align}
% \vspace{-0.3cm}
with 
% \vspace{-0.2cm}
\begin{align*}
    \bd{e}_i = & \ \bd{s}_{t_i}- \bd{M}_{i-1}\bd{u}_{t_i}\\
    \bd{P}_i = & \frac{1}{\lambda} \left( \bd{P}_{i-1} - \frac{\bd{P}_{i-1} \bd{u}_{t_i}  \bd{u}_{t_i}^\top \bd{P}_{i-1}}{ \lambda + \bd{u}_{t_i}^\top \bd{P}_{i-1} \bd{u}_{t_i}} \right)\\
    \gamma_i = & \lambda \gamma_{i-1} + 1.
\end{align*}
%where $\bd{M}_0 = \bd{0}$, $\boldsymbol{\Sigma}_0 = \bd{0}$, $\bd{P}_0 = \bd{I}_{d}$ and $\gamma_0 = 0$.
\end{theorem}
\begin{proof}
    See Appendix \ref{ap:proof_theorem1}.
\end{proof}

The equations above describe how to update the model parameters at each time $t$ and for each calendar type $c$. Such parameters are updated adding corrections to the previous parameters $\bd{M}_{i-1}$ and $\boldsymbol{\Sigma}_{i-1}$. These corrections are given by the fitting error $\bd{e}_{i}$ of the previous parameter, so that parameters are updated based on how they fit to the most recent data. In addition, the  recursive equations utilize matrix $\bd{P}_i$ and variable $\gamma_i$ commonly referred in adaptive recursive methods as the data autocorrelation matrix and cumulative forgetting factor \cite{liu2011kernel}. Matrix  $\bd{P}_i$  accounts for the uncertainties in the parameters estimation, while  variable $\gamma_i$ describes the rate of change for the covariance matrix. With the recursive equations \eqref{eq:mean_MT} and \eqref{eq:sigma_MT} the proposed methods obtain probabilistic models characterized by dynamic means and covariances. 

Theorem \ref{theorem1} generalizes the results obtained in \cite{alvarez2021probabilistic} for single-task load forecasting. The single-task learning method obtains a model for each entity, while the multi-task learning method obtains a joint model for the $K$ entities. In fact, the matrix of means $\bd{M}$ models a linear combination of the $K$ entities at each time $t$. In addition, the covariance matrix $\boldsymbol{\Sigma}$ not only contains load variance as in single-task learning, but also provides the correlation among different tasks. While the single-task method learns the mean parameters $\boldsymbol{\eta}_{s,c}\in \mathbb{R}^2$, $\boldsymbol{\eta}_{r,c}\in \mathbb{R}^R$ and standard deviations $\sigma_{s,c},\sigma_{r,c}\in\mathbb{R}$, the multi-task method learns a higher number of parameters for the mean matrices $\bd{M}_{s,c} \in \mathbb{R}^{K\times(K+1)}$, $\bd{M}_{r,c} \in \mathbb{R}^{K \times KR}$, and for the covariance matrices $\bd{\Sigma}_{s,c}, \bd{\Sigma}_{r,c} \in \mathbb{R}^{K\times K}$.

The covariance matrix captures relationships among entities, but estimating it accurately can be challenging, particularly with limited data. Certain elements of the covariance matrix~$\boldsymbol{\Sigma}_i$ can be set to zero to avoid spurious correlations due to the finite sample size. A common approach for obtaining sparse matrices is the thresholding method, which sets to zero small off-diagonal elements of the covariance matrix. However, in certain cases, the resulting covariance matrix may lose the positive definiteness \cite{guillot2012retaining}. In Section \ref{sec:experiments}, we use a simple method that ensures the resulting covariance matrix remains positive definite.

% \vspace{-0.2cm}
\subsection{Prediction}
The previous section details how to update HMM parameters using the most recent data recursively. The following theorem shows how to obtain probabilistic forecasts using these parameters.

\begin{theorem} 
\label{theorem}
Let $\left\{\bd{s}_t, \bd{r}_t \right\}_{t\geq 1}$ be an HMM characterized by parameters $\Theta$ defined in \eqref{eq:theta_parameters}. Then, for $i=1,2,...,L$
\begin{equation}
\label{eq:theorem}
    p(\bd{s}_{t+i}|\bd{s}_t, \bd{r}_{t+1:t+i}) = \set{N}(\bd{s}_{t+i}; \widehat{\bd{s}}_{t+i}, \widehat{\bd{E}}_{t+i})
\end{equation}
where mean $\widehat{\bd{s}}_{t+i}$ and covariance matrix $\widehat{\bd{E}}_{t+i}$ can be recursively obtained by
\begin{align}
    \widehat{\bd{s}}_{t+i}  = & \bd{W}_1 (\bd{W}_1+ \bd{W}_2)^{-1}\bd{M}_{r,c}\bd{u}_r  + \nonumber \\
    &  \bd{W}_2 (\bd{W}_1+\bd{W}_2)^{-1} \bd{M}_{s,c}\widehat{\bd{u}}_{s} \label{eq:pred_mean}\\
   \widehat{\bd{E}}_{t+i}  = &  \bd{W}_2 (\bd{W}_1 +\bd{W}_2)^{-1}\bd{W}_1 \label{eq:pred_error}
\end{align}
with 
\begin{align*}
\bd{W}_1 = &\bd{\Sigma}_{s,c}+ \bd{M}_{s,c}\bd{N} \widehat{\bd{E}}_{t+i-1}(\bd{M}_{s,c}\bd{N})^\top , \ \ \bd{W}_2 = \bd{\Sigma}_{r,c}, \nonumber\\ 
c  =  c(t+&i), \widehat{\bd{u}}_s = [1, \widehat{\bd{s}}_{t+i-1}^\top]^\top, \bd{u}_r = u_r(\bd{r}_{t+i}), 
\bd{N} =\renewcommand{\arraystretch}{0.5} % Ajuste la separación vertical entre filas
\begin{bmatrix}
\bd{0} \\
\bd{I}_K
\end{bmatrix} \nonumber
\end{align*}
% \vspace{-0.1cm}
and initial values given by $\widehat{\bd{s}}_t = \bd{s}_t$ and $\widehat{\bd{E}}_t = \bd{0}$.  
\end{theorem}

\begin{proof}
    See Appendix \ref{ap:proof_theorem2}.
\end{proof}
 The above theorem enables to recursively obtain probabilistic forecasts $\set{N}(\bd{s}_{t+i}; \widehat{\bd{s}}_{t+i}, \widehat{\bd{E}}_{t+i})$ for $i = 1, 2, \ldots, L$ that allow us to quantify the probability of forecast intervals (see Fig.~\ref{fig:prediction_intervals}). Theorem~\ref{theorem} provides load forecasts $\widehat{\bd{s}}_{t+i}$ and the estimates of their uncertainty $\widehat{\bd{E}}_{t+i}$, for \mbox{$i=1,2,...,L$}. Such forecasts are obtained using the recursions \eqref{eq:pred_mean} and \eqref{eq:pred_error} with the most recent parameters every time new instance vectors $\bd{x}$ are obtained. Specifically, the probabilistic forecasts $\set{N}(\bd{s}_{t+i}; \widehat{\bd{s}}_{t+i}, \widehat{\bd{E}}_{t+i})$ at time $t+i$ for $i=1,2,...,L$ are obtained using the probabilistic forecasts $\set{N}(\bd{s}_{t+i-1}; \widehat{\bd{s}}_{t+i-1}, \widehat{\bd{E}}_{t+i-1})$ at previous time, observations vectors $\bd{r}_{t+i}$ at time $t+i$ and the model parameters corresponding to the calendar type at time $t+i$, denoted as $c=c(t+i)$. 

\begin{figure}
\centering
        \centering
        \psfrag{data1}[l][l][0.5]{$\widehat{\bd{s}}_t \pm 4\widehat{\bd{E}}_{t}$}
        \psfrag{data2}[l][l][0.5]{$\widehat{\bd{s}}_t \pm 2\widehat{\bd{E}}_{t}$}
        \psfrag{data3}[l][l][0.5]{Load forecast $\hat{\bd{s}}_t$}
        \psfrag{abcdefghijklmnopqrstuvwxyz}[l][l][0.5]{Load $\bd{s}_t$}
        \psfrag{Error}[t][][0.7]{Error [MW]}
        \psfrag{0}[][][0.7]{0}        
        \psfrag{8}[][][0.6]{8\hspace{2mm}}        
        \psfrag{10}[][][0.6]{10\hspace{2mm}}
        \psfrag{12}[][][0.6]{12\hspace{2mm}}
        \psfrag{14}[][][0.6]{14\hspace{2mm}} 
        \psfrag{16}[][][0.6]{16\hspace{2mm}}   
        \psfrag{24}[][][0.6]{24} 
        \psfrag{48}[][][0.6]{48}   
        \psfrag{72}[][][0.6]{72}   
        \psfrag{96}[][][0.6]{96}   
        \psfrag{120}[][][0.6]{120}   
        \psfrag{144}[][][0.6]{144}   
        \psfrag{hours}[t][][0.7]{Hours [h]}   
        \psfrag{loadconsumption}[b][][0.7]{Load consumption [GW]}   
        \includegraphics[width=0.5\textwidth]{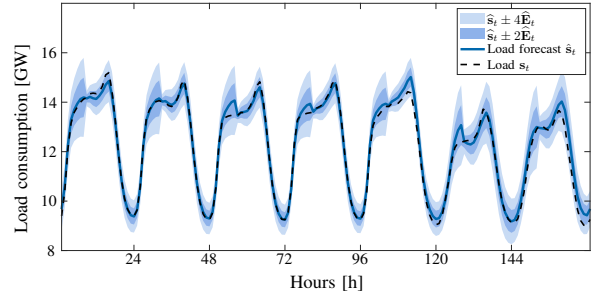}
    \caption{\mbox{Multi-APLF} method obtains load forecasts together with reliable uncertainty assessments of load demand.}
    \label{fig:prediction_intervals}
\end{figure}

Single-task forecasts in \cite{alvarez2021probabilistic} depend on each entity’s model and past predictions, while the presented multi-task approach incorporates models and predictions from multiple entities to capture dependencies among them. Specifically, as shown in~\eqref{eq:pred_mean}, the load forecast $\widehat{\bd{s}}_{t+i}$ is given by the linear combination of $\bd{M}_{r,c}\bd{u}_r$ and $\bd{M}_{s,c}\widehat{\bd{u}}_s$ that are given by models and previous predictions for all entities.  In addition, the weights $\bd{W}_1$ and $\bd{W}_2$ are given by the relationships among entities represented by the covariance matrices $\boldsymbol{\Sigma}_{s,c}$ and $\boldsymbol{\Sigma}_{r,c}$. Such matrices $\bd{W}_1$ and $\bd{W}_2$ balance the information gathered from the new observations (term $\widehat{\bd{u}}_r$) and the previous prediction (term $\widehat{\bd{u}}_s$). For instance, when $\bd{W}_1 \gg \bd{W}_2$, the resulting prediction is dominated by the observation-based term, since $\bd{W}_1(\bd{W}_1 + \bd{W}_2)^{-1} \approx \bd{I}$ and $\bd{W}_2(\bd{W}_1 + \bd{W}_2)^{-1} \approx \bd{0}$. The covariance matrix $\widehat{\bd{E}}_{t+i}$ describes the estimated uncertainty in the forecasts and how forecasts for one entity are influenced by other entities. In particular, each diagonal term $(\widehat{\bd{E}}_{t+i})_{j,j}$ describes the variance for the forecasts for the $j$-th entity, while each non-diagonal term $(\widehat{\bd{E}}_{t+i})_{j,k}$ describes the correlation between the forecasts for the $j$-th and $k$-th entities. 

The theoretical results above describe the learning and prediction steps of the proposed method. The subsequent section details the implementation steps using the proposed techniques.

 % \vspace{-0.3cm}
\subsection{Implementation}

The proposed learning methods continuously update models with the most recent data samples to adapt to changes in consumption patterns over time. In particular, the proposed method recursively updates the model parameters with the most recent data using \eqref{eq:mean_MT} and \eqref{eq:sigma_MT}, and provides probabilistic forecasts using the recursions defined in Theorem~\ref{theorem}.

\begin{figure}
\centering
        \includegraphics[width=0.45\textwidth]{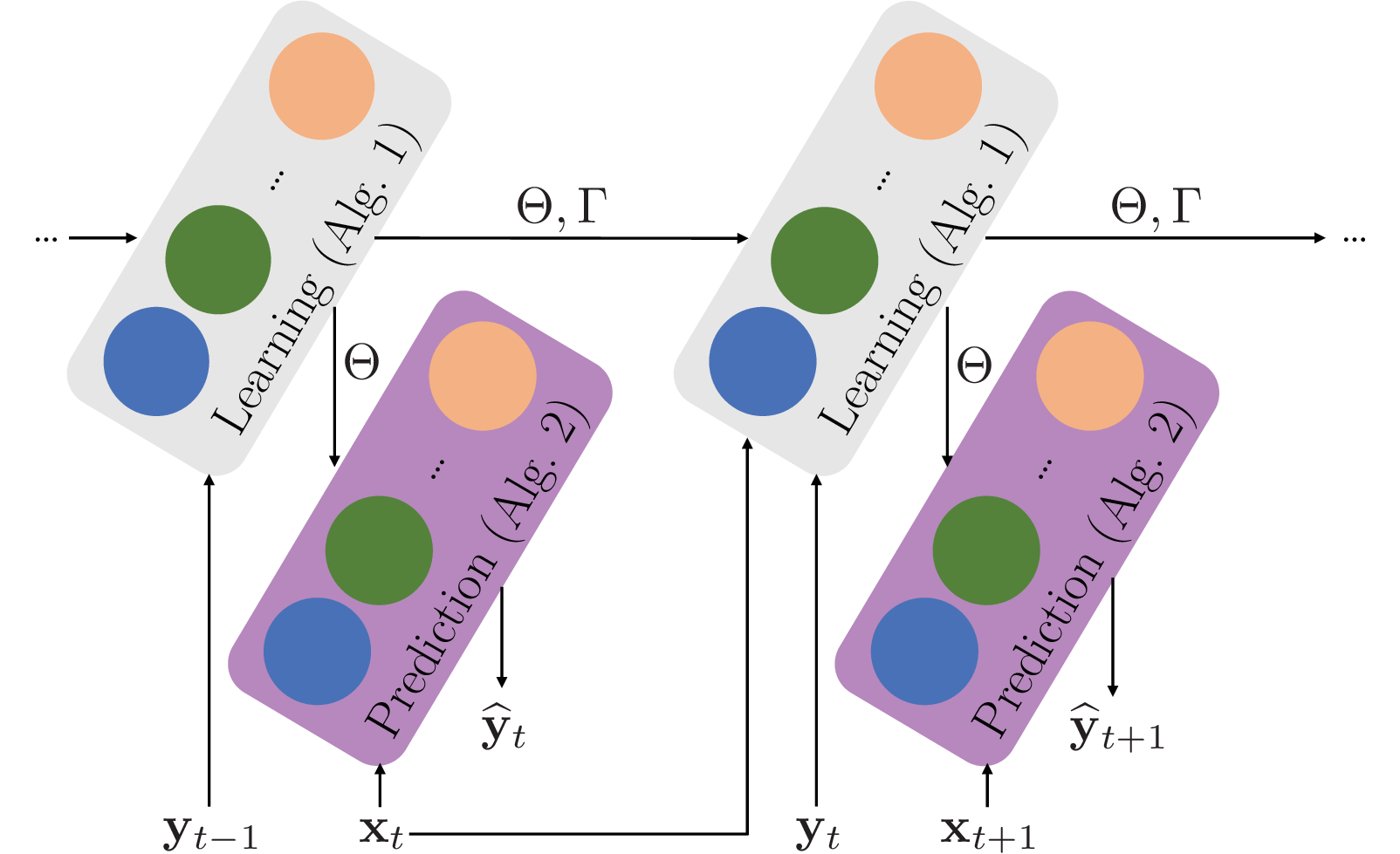}
    \caption{Diagram of the proposed Multi-APLF methods.}
    \label{fig:diagram}
\end{figure}

 Figure~\ref{fig:diagram} depicts the flow diagram of the proposed method that iteratively performs the learning and prediction steps. The proposed methods train
models regularly using the most recent samples from multiple entities. Then, such methods predict for each new instance vector using the latest learned model. At learning, training samples from multiple entities and the previous model are used to obtain a new model. At prediction, the latest learned model and the instance vector $\bd{x}_t$ are used to obtain load forecasts $\widehat{\bd{y}}_t$ at time $t$. 

The learning step is described in Algorithm \ref{alg:ALG1}. The method updates the new model using the instance vector of $K$ entities $\bd{x}_t$, the actual loads of $K$ entities $\bd{y}_t$, and the previous model. Algorithm \ref{alg:ALG1} updates  parameters $\bd{M}_{s,c}, \boldsymbol{\Sigma}_{s,c}, \bd{M}_{r,c}, \boldsymbol{\Sigma}_{r,c}$ as well as state variables $\bd{P}_{s,c}, \gamma_{s,c}, \bd{P}_{r,c}, \gamma_{r,c}$ with recursions given in \eqref{eq:mean_MT} and \eqref{eq:sigma_MT}.

\begin{algorithm}[h!]
        \caption{Learning step}
        \label{alg:ALG1}
        \begin{algorithmic}
        \Require{ 
\begin{tabular}[t]{ll}
      $\Theta = \{ \bd{M}_{s,c}, \boldsymbol{\Sigma}_{s,c}, \bd{M}_{r,c}, \boldsymbol{\Sigma}_{r,c} \}$ & model parameters \\
      $\Gamma = \{ \bd{P}_{s,c}, \gamma_{s,c}, \bd{P}_{r,c}, \gamma_{r,c} \}$ & state variables \\
      $\lambda_s, \lambda_r$ & forgetting factors \\
      $\bd{x}_t$ & new instance \\
      $\bd{y}_t$ & new loads vector \\
      $L$ & prediction horizon\\
      $t$ & time
    \end{tabular} \qquad}
    
        \Ensure{ \begin{tabular}[t]{ll}
      $\Theta$ & updated parameters \\
      $\Gamma$ & updated state variables 
    \end{tabular} \qquad}
    \vskip 4pt
            \ForAll{ $i = 1,2,...,L$}
            \vskip 2pt
            \State {$c \gets c(t+i)$ }
            \vskip 2pt
            \State {$\bd{u}_s \gets [1, \bd{s}_{t+i-1}^\top] ^\top$}
            \vskip 2pt
            \State {$\bd{u}_r \gets u_r(\bd{r}_{t+i})$}
            \vskip 2pt
            \ForAll  {$j=s,r$}
            \State {$\bd{e}_{j,c} \gets \bd{s}_{t+i} - \bd{M}_{j,c} \bd{u}_j$}
            \State    $\bd{P}_{j,c} \gets \dfrac{1}{\lambda_j} \left( \bd{P}_{j,c} -  \dfrac{\bd{P}_{j,c} \bd{u}_j \bd{u}_j^\top \bd{P}_{j,c}}{ \lambda + \bd{u}_j^\top \bd{P}_{j,c} \bd{u}_j} \right)$
            \State    $\gamma_{j,c} \gets \lambda_j \bd{\gamma}_{j,c} + 1$
            \State    $\bd{\Sigma}_{j,c} \gets  \bd{\Sigma}_{j,c} - \dfrac{1}{\bd{\gamma}_{j,c}} \Big( \bd{\Sigma}_{j,c}  - \dfrac{\lambda_j^2  \bd{e}_{j,c} \bd{e}_{j,c}^\top }{(\lambda_j+\bd{u}_j^\top\bd{P}_{j,c}\bd{u}_j)^2} \Big)$
            \State    $\bd{M}_{j,c}\gets \bd{M}_{j,c} + \dfrac{\bd{e}_{j,c} \bd{u}_j^\top \bd{P}_{j,c}}{\lambda_j + \bd{u}_j^\top \bd{P}_{j,c}^\top\bd{u}_j}$
            \EndFor 
            \EndFor
        \end{algorithmic}
    \end{algorithm}

The prediction step is described in Algorithm \ref{alg:ALG2}. The method uses the latest learned model to obtain load forecasts for $K$ entities $\widehat{\bd{y}}_t = [ \widehat{\bd{s}}_{t+1}^\top, \widehat{\bd{s}}_{t+2}^\top, ..., \widehat{\bd{s}}_{t+L}^\top]^\top$, and estimated covariances $\widehat{\bd{E}}_{t+1}, \widehat{\bd{E}}_{t+2}, ..., \widehat{\bd{E}}_{t+L}$ that determine probabilistic forecasts $p(\bd{s}_{t+i}|\bd{s}_t, \bd{r}_{t+1:t+i})$ for $i=1,2,...,L$.

\begin{algorithm}
        \caption{Prediction step}
        \label{alg:ALG2}
        \begin{algorithmic}
        \Require{ 
\begin{tabular}[t]{ll}
      $\Theta$ & \  model parameters \\
      $\bd{x}_t = [\bd{s}_t^\top, \bd{r}_{t+1}^\top,\bd{r}_{t+2}^\top, ..., \bd{r}_{t+L}^\top ]^\top $ & \ new instance \\
      $t$ & \ time
    \end{tabular} \qquad}
        \Ensure{ \begin{tabular}[t]{ll}
      $\widehat{\bd{y}}_t = [\widehat{\bd{s}}_{t+1}^\top,\widehat{\bd{s}}_{t+2}^\top, ..., \widehat{\bd{s}}_{t+L}^\top]^\top$ & load forecast \\
      $[\widehat{\bd{E}}_{t+1},\widehat{\bd{E}}_{t+2}, ..., \widehat{\bd{E}}_{t+L} ]^\top$ &  error covariances\\
      $\{ \set{N}(\bd{s}_{t+i}; \widehat{\bd{s}}_{t+i}, \widehat{\bd{E}}_{t+i})\}_{i=1}^L$ & prob. forecasts
    \end{tabular} \qquad}
            \State $\widehat{\bd{s}}_t \gets \bd{s}_t $
            \vskip 2pt
            \State $\widehat{\bd{E}}_t \gets \bd{0}$
            \vskip 2pt
            \State $\bd{N} \gets 
            \begin{small}
                \renewcommand{\arraystretch}{0.4} % Ajuste la separación vertical entre filas
                \begin{bmatrix}
                 \bd{0} \\
                 \bd{I}_K
                \end{bmatrix}
                \end{small}$
                \vskip 4pt
            \ForAll{ $i = 1,2,...,L$}
            \vskip 2pt
            \State {$c \gets c(t+i)$ }
            \vskip 2pt
            \State {$\hat{\bd{u}}_s \gets [1, \widehat{\bd{s}}_{t+i-1}^\top] ^\top$}
            \vskip 2pt
            \State {$\bd{u}_r \gets u_r(\bd{r}_{t+i})$}
            \begin{flalign*}
            \hspace{0.4cm} \bd{W}_1  \gets & \bd{\Sigma}_{s,c}+ \bd{M}_{s,c}\bd{N} \widehat{\bd{E}}_{t+i-1}(\bd{M}_{s,c}\bd{N})^\top & \\
            \hspace{0.4cm}\bd{W}_2 \gets&  \bd{\Sigma}_{r,c}&\\
            \hspace{0.4cm}  \widehat{\bd{s}}_{t+i}  \gets & \bd{W}_1 (\bd{W}_1+ \bd{W}_2)^{-1}\bd{M}_{r,c}\bd{u}_r  +&\\
            \hspace{0.4cm}&  \bd{W}_2 (\bd{W}_1+\bd{W}_2)^{-1} \bd{M}_{s,c}\widehat{\bd{u}}_{s}& \\
           \hspace{0.4cm}  \widehat{\bd{E}}_{t+i} \gets & \bd{W}_2 (\bd{W}_1 +\bd{W}_2)^{-1}\bd{W}_1&
            \end{flalign*}
            \EndFor 
        \end{algorithmic}
    \end{algorithm}

\begin{table}[b]
 % \vspace{-0.4cm}
\caption{Complexity of existing techniques.}
\label{tab:complexity}
\setlength{\tabcolsep}{3pt} % Reduce column spacing
\renewcommand{\arraystretch}{1.2} % Adjust row height
\resizebox{\columnwidth}{!}{%
\begin{tabular}{lcccc}
\toprule
{Method} & {Learning} &  {Prediction} &  {Memory} \\ \hline
 {MTGP} & $O(K^3N^3)$ & $O(LK^2N^2)$ & $O(K^2N^2)$ \\
 {VAR}  & $O(K^3M^3 + K^2M^2N)$ & $O(LK^2M^2)$ & $O(K^2M)$ \\
 {MLR}  & $O(K^3M^3 + K^2M^2N)$ & $O(LK^2M)$ & $O(K^2M)$ \\
 {\mbox{Multi-APLF}} & $O(LK^2)$ & $O(LK^3R^3)$ & $O(C(K^2+ R^2))$ \\
\bottomrule
\end{tabular}%
}
\end{table}
The running times of Algorithm~1 and Algorithm~2 are suitable for real-time implementation with very low latency. Table~\ref{tab:complexity} compares the computational and memory complexity of the proposed \mbox{Multi-APLF} with the multi-task learning methods MTGP~\cite{Zhang2014}, MLR, and the Vector Autoregressive model~\cite{lutkepohl2005new} (VAR). The proposed \mbox{Multi-APLF} method is designed for online adaptive learning, where the model is updated incrementally at each time step. While this requires continual model updates, each update has low computational cost, with complexity that grows only quadratically with the number of entities $K$, and remains constant with respect to the number of historical load data $N$. In contrast, baseline methods such as MTGP, MLR, and VAR rely on offline training, fitting a static model using a large training set of size $N$, and using it unchanged for future predictions. These offline methods have significantly higher training costs, with computational complexity that scales cubically with both $K$ and $N$ (e.g., $O(K^3N^3)$ for MTGP) and memory complexity up to $O(K^2N^2)$ due to kernel matrix storage. Although offline training is performed only once, maintaining predictive accuracy over time typically requires frequent retraining to adapt to nonstationary conditions. In contrast, \mbox{Multi-APLF}'s lightweight updates allow efficient adaptation without retraining, making it substantially more scalable in real-time and large-scale multitask forecasting scenarios.

The contents presented in this section detail the implementation of the proposed method, including the specific algorithms used and a comparison of the computational and memory complexity with existing methods. The corresponding Python source code is publicly available in \url{https://github.com/MachineLearningBCAM/Multitask-load-forecasting-IEEE-TPWRS-2025}. The following section presents the experimental results that show the performance improvement achieved by the proposed techniques.

\section{Experiments}
\label{sec:experiments}

 This section first describes the datasets used for
the numerical results and the experimental setup, and then
compares the proposed method’s performance with that of
existing techniques.

\subsection{Datasets and experimental setup}
\label{experimental_setup}
Five publicly available load consumption datasets are used in the numerical experiments: a dataset from the Global Energy Forecasting Competition 2017 (GEFCom) with load demand of eight regions and weather information \cite{hong2019global}; a dataset from New England made available by the ISO New England organization \cite{ISONE} that registers hourly load demand, electricity cost information, weather data, and system load for the ISO New England Control Area and its eight load zones in 2022; a dataset from a regional transmission organization in the United States (PJM Interconnection LLC \cite{PJM2023}), which collects load consumption data from eight regions in 2023;  Australian Electricity Demand dataset, \textcolor{black}{which compiles the load demand for five Australian cities \cite{datasetAustralia}, depicted in Figure~\ref{fig:patterns}};  and a dataset from New South Wales \cite{dataset400buildings} that registers load demand for 400 buildings in 2013.

\begin{table*}[]
\caption{RMSE [GW] and MAPE [\%] for prediction errors for the proposed method and state-of-the-art techniques.}
% \small
\label{tab:experiments}
\resizebox{\textwidth}{!}{%
\begin{tabular}{ccl|rr|rr|rr|rr|rr|rr}
\toprule
 &     &     & \multicolumn{2}{c|}{{APLF}}& \multicolumn{2}{c|}{{N-HITS}} & \multicolumn{2}{c|}{{MTGP}}& \multicolumn{2}{c|}{{VAR}}& \multicolumn{2}{c|}{{MLR}}& \multicolumn{2}{c}{{\mbox{Multi-APLF}}}\\
&      &    & \multicolumn{1}{c}{{[\%]}} & \multicolumn{1}{c|}{{[GW]}} & \multicolumn{1}{c}{{[\%]}} & \multicolumn{1}{c|}{{[GW]}} & \multicolumn{1}{c}{{[\%]}} & \multicolumn{1}{c|}{{[GW]}} & \multicolumn{1}{c}{{[\%]}} & \multicolumn{1}{c|}{{[GW]}} & \multicolumn{1}{c}{{[\%]}} & \multicolumn{1}{c|}{{[GW]}}& \multicolumn{1}{c}{{[\%]}} & \multicolumn{1}{c}{{[GW]}} \\ 
\midrule
\multirow{9}{*}{\rotatebox{90}{{GEFCom}}} & \multirow{8}{*}{\rotatebox{90}{{Entity}}} & 1 & 5.67 & 0.26 & 7.52 & 0.35 & 6.82 & 0.30 & 7.07 & 0.34 &  11.01 & 0.62 & \textbf{5.32} & \textbf{0.25}\\
 & & 2 & 3.63 & 0.06 & 4.77 & 0.07 & 4.52& 0.07& 5.28 & 0.09& 6.87 & 0.13  & \textbf{3.40}& \textbf{0.05}\\
& & 3 & 5.09 & 0.19 & 6.72 & 0.25  & 6.15 & 0.24& 6.97 & 0.26 & 8.92 & 0.40 & \textbf{4.49}& \textbf{0.17}            \\
& & 4 & 4.62 & 0.08  & 6.52 & 0.11  & 5.95 & 0.11& 6.45  &   0.11 & 9.83 & 0.21& \textbf{4.19}& \textbf{0.07}            \\
&  & 5 & 5.29 & \textbf{0.06} & 6.83 & 0.08 & 5.96 & 0.07&  6.65  & 0.08 &  10.99 &  0.16  & \textbf{4.76}& \textbf{0.06}            \\
&   &  6 & 5.87 & 0.13 & 7.16 & 0.16 & 6.34& 0.14& 7.22  & 0.16 & 12.75 & 0.35  & \textbf{5.21}& \textbf{0.12}            \\
&  &  7 & 5.11 & 0.04 & 6.05 & 0.04 & 5.63 & 0.04 & 6.48 & 0.05& 8.66 & 0.07  & \textbf{4.84}& \textbf{0.03}            \\
& &  8 & 5.32 & 0.13 & 6.29  & 0.16  & 6.78 & 0.16 &  7.18 & 0.19 & 9.63 & 0.29 & \textbf{4.70}& \textbf{0.12}            \\
% \cline{2-15} 
& \multicolumn{2}{c|}{{Average}}& 5.08 & 0.12 & 6.48 & 0.15 & 6.38 & 0.15 &  6.70 & 0.17 & 9.83  & 0.28  & \textbf{4.61}& \textbf{0.11}            \\
\hline
\multirow{9}{*}{\rotatebox{90}{{NewEngland}}} & \multirow{8}{*}{\rotatebox{90}{{Entity}}} &  1 & 10.19 & 0.13 & 10.47 & 0.14 &11.69 & 0.15 &  12.91 & 0.16 & 16.15 & 0.23  & \textbf{9.73}& \textbf{0.12}            \\
 & & 2 & 6.42 & 0.10 & 7.32 & 0.11  & 8.23 & 0.12 &  7.44  & 0.11 & 11.50 & 0.20 & \textbf{5.77}& \textbf{0.09}            \\
 & & 3 & 12.35 & 0.07 & 16.30 & 0.09 & 16.35 & 0.12 &  19.68   & 0.10 & 31.51 & 0.15  & \textbf{11.90}& \textbf{0.06}    \\
 & & 4 & 6.57 & 0.27 & 7.33 & 0.31 & 9.23 & 0.37 &  7.63&  0.32 & 12.72 & 0.61  & \textbf{6.19} & \textbf{0.25}            \\
& & 5 & 8.61 & 0.09 &  9.54 & 0.10 & 10.53 & 0.11 & 9.62 &  0.11 & 15.79 & 0.20  & \textbf{7.72}& \textbf{0.08}            \\
& & 6 & 7.69 & \textbf{0.15}& 8.36 & 0.16 & 11.95 & 0.24 & 8.85 &  0.17 & 14.81 & 0.34  & \textbf{7.11}& \textbf{0.15}            \\
& & 7 & 9.17 & 0.23 & 8.95 & 0.23  & 9.44 & 0.24 & 10.32  & 0.26  & 18.31 & 0.50  & \textbf{7.91} & \textbf{0.20}            \\
  & &  8 & 5.66 & 0.19 & 6.42 & 0.22  & 8.21 & 0.26 & 6.93 & 0.24 & 11.41 & 0.42 & \textbf{5.20} & \textbf{0.17}            \\
% \cline{2-15} 
 & \multicolumn{2}{c|}{{Average}}    & 8.33 & 0.15 & 9.34 & 0.17 & 10.70 & 0.20 &  10.42 & 0.18 & 16.53 & 0.33   & \textbf{7.69}& \textbf{0.14}            \\ 
\hline
\multirow{9}{*}{\rotatebox{90}{{PJM}}}& \multirow{8}{*}{\rotatebox{90}{{Entity}}} &  1 & 4.84  & 0.91 & 5.52 & 1.04 &  6.39 & 1.13 & 6.15 & 1.09 & 6.16 &  1.11  & \textbf{4.46}  & \textbf{0.85}           \\
 && 2  & 5.27 & 0.79 & 6.41 &  0.99 & 7.04 & 0.99 & 9.20 &  1.15 & 7.21 & 0.97   &  \textbf{5.15} & \textbf{0.76}  \\
& & 3 & 5.75 & 0.14  & 6.94 & 0.18   & 7.83 & 0.18 & 9.46 & 0.22  & 7.16 & 0.17  & \textbf{5.23} &  \textbf{0.12}   \\
 & &  4 & 5.83 & 0.22 & 6.79 & 0.27 & 7.22  & 0.26 & 9.20 & 0.34  & 7.32 & 0.27   &  \textbf{5.37} & \textbf{0.20}        \\
& &  5  & 5.45 & 0.97  & 6.19 & 1.18 & 7.50 & 1.26 & 6.46 &  1.15 & 6.36 & 1.15   & \textbf{5.38} &  \textbf{0.95}    \\
& &  6  & 5.01 & 0.10 & 5.48 &  0.11 & 5.21 & 0.10 & 6.44 & 0.12  & 6.41 & 0.13  & \textbf{4.63} &  \textbf{0.09}      \\
& &  7  & 9.53 & 0.19 & 10.45 & 0.20  & 12.53 & 0.22 & 24.64  & 0.48  & 22.51 & 0.44  & \textbf{9.15} &  \textbf{0.18}     \\
& & 8  & 4.45 & 0.44 & 5.34 & 0.54  & 5.90 & 0.56 & 6.62 & 0.63  & 5.80 & 0.59  & \textbf{4.08} &   \textbf{0.40}    \\
% \cline{2-15} 
& \multicolumn{2}{c|}{{Average}}     & 5.77 & 0.47 & 6.64 &  0.55  & 7.42 & 0.58 & 9.99 &  0.69 & 8.62 & 0.60 & \textbf{5.43} &  \textbf{0.45}     \\   
\hline
\multirow{6}{*}{\rotatebox{90}{{Australia}}} & \multirow{5}{*}{\rotatebox{90}{{Entity}}} &  1 & 4.16 & 0.74 & 4.94 & 0.87 & 7.52 & 1.26 & 7.75 & 1.40 & 8.05 & 1.34  & \textbf{3.97}  &\textbf{0.71}      \\
& & 2  & 3.96 & 0.48 & 5.35 & 0.67 & 7.89 & 0.88 & 8.32 & 1.04 & 7.45 & 0.85  &   \textbf{3.69}     &  \textbf{0.45}    \\
& & 3 & 3.12  & 0.35 & 3.66 & 0.40 & 6.41 & 0.69 & 7.10 & 0.82 & 7.63 & 0.74  & \textbf{3.04} &  \textbf{0.34}\\
& & 4 & 5.03 &  0.18 & 6.86 & 0.24 & 10.70 & 0.32  & 9.99 & 0.36 & 14.17 & 0.43  &    \textbf{4.83}   &  \textbf{0.18} \\
 & & 5 & 7.22 & \textbf{0.09} & 7.47 & 0.10  & 10.19 & 0.13 & 9.28 & 0.13 & 11.54 & 0.15 &    \textbf{7.01}   &  \textbf{0.09} \\
% \cline{2-15} 
&\multicolumn{2}{c|}{{Average}}   & 4.70 & 0.37 & 5.65 & 0.45  & 8.54 & 0.66  & 8.38 & 0.75 & 9.77 & 0.70 &\textbf{4.51} &    \textbf{0.35}   \\ \hline
\end{tabular}%
}
\end{table*}

The prediction for all the algorithms is done daily at 11~a.m., and the algorithms obtain future loads for a prediction horizon of $L=24$ hours. Therefore, every vector of load forecasts consists of $24$ hourly predictions. For simplicity, all the methods use the default values for hyper-parameters. For the methods based on offline learning, we use the initial 30 days of data for training. Such training size has been selected based on MTGP method \cite{Zhang2014}, which can hardly use larger training sizes due to its high learning complexity (see also Table ~\ref{tab:complexity}).

The proposed method is implemented as follows. The instance vector $\bd{x}_t~=~[\bd{s}_t^\top, \bd{r}_{t+1}^\top, \bd{r}_{t+2}^\top,...,\bd{r}_{t+L}^\top]^\top$ is composed of the loads at time $t$ and observations corresponding to the next $L$ time steps. The observations vector \mbox{$\bd{r}_t=[\bd{r}_{t,1}^\top, \bd{r}_{t,2}^\top, ..., \bd{r}_{t,K}^\top]^\top$} contains information about the temperature $w_t$ at time $t$ and the mean of past temperatures $\overline{w}_{c(t)}$ in each entity. The observations vector $\bd{r}_{t,k}$ for each entity $k=1,2,...,K$ is represented by the feature vector $u_r(\bd{r}_{t,k})$ as in \cite{alvarez2021probabilistic}, which is defined in terms of temperature shifts. 

The calendar information $c(t)$ ranges from 1 to 24 to denote the hours of weekdays, and from 25 to 48 to denote the hours of weekends and holidays, that is, \mbox{$c(t) \in \{1, 2, ..., 48\}$}. Therefore, the proposed method obtains parameters $\bd{M}_{s,c}, \boldsymbol{\Sigma}_{s,c}, \bd{M}_{r,c}, \boldsymbol{\Sigma}_{r,c}$ for each calendar type \mbox{$c(t) \in \{1,2, ...,48\}$}, which are updated using forgetting factors set to $\lambda_s = 0.9$ and $\lambda_r = 0.9$ in all calendar types and datasets. The initialization of the parameters of the model is set to $\bd{M}_0 = \bd{0}$, $\boldsymbol{\Sigma}_0 = \bd{0}$, $\bd{P}_0 = \bd{I}_{d}$, and $\gamma_0 = 0$.

As described in Section \ref{sec:multitask}, we use a thresholding technique that sets to zero small off-diagonal elements to effectively remove these weak covariances while preserving the positive definiteness of the covariance matrix. For each component $(i,j)$ of the matrix, the correlation is computed as \mbox{$\rho_{i,j} = \mathbf{\Sigma}_{i,j}/\sqrt{\mathbf{\Sigma}_{i,i} \cdot \mathbf{\Sigma}_{j,j}}$}, and off-diagonal elements of $\boldsymbol{\Sigma}$ are set to $0$ whenever $|\rho_{i,j}|$ is below a predefined threshold. To ensure the matrix remains positive definite after thresholding, we compensate by adding rank-1 matrices of the form $\mathbf{v}\mathbf{v}^{\top}$, where the vector $\mathbf{v}$ is given as \mbox{$v_i =\sqrt{|\boldsymbol{\Sigma}_{i,j}|},$} $ v_j = -v_i, \hspace{0.2cm} \text{and} \hspace{0.1cm} v_k=0$, for any $i,j,k$ components such that  $k \neq i,j$. Note that vector $\mathbf{v}$ has nonzero entries only at the positions corresponding to the removed covariances. Since such rank-1 matrices are positive semi-definite, this correction helps restore and preserve the positive definiteness of the full covariance matrix. This approach is computationally efficient and results in a sparse, positive-definite structure that enhances prediction robustness, particularly in data-scarce settings. In all the experimental results, we use the threshold of 0.1 as the minimum non-zero correlation.

The proposed method is compared against five state-of-the-art techniques: two single-task methods (APLF~\cite{alvarez2021probabilistic} and N-HiTS~\cite{challu2023nhits}) and three multi-task methods (MTGP~\cite{Zhang2014}, VAR model with exogenous variables~\cite{lutkepohl2005new}, and MLR). APLF method is implemented as in~\cite{alvarez2021probabilistic}. The technique N-HiTS~\cite{challu2023nhits} is a single-task deep learning method designed to capture consumption patterns across different frequencies and scales.

The MTGP method is implemented as described in~\cite{Zhang2014} and is a multi-task probabilistic method that can capture non-linear relationships in load consumption data from multiple entities. Finally, the VAR model with exogenous variables~\cite{lutkepohl2005new} and MLR estimate future loads of multiple entities as linear combinations of past consumption values and seasonal exogenous variables, such as the day of the week and the hour of the day. %with $m=5$ as the lag value.

% \vspace{-0.3cm}
\subsection{Numerical results}

The first set of experiments quantifies the prediction error of the proposed method in comparison with the 6 existing techniques for 4 datasets. Table~\ref{tab:experiments} shows the RMSE and MAPE, assessing the overall load prediction errors of the six methods. These results show that the proposed method achieves higher accuracy than existing multi-task learning algorithms. Simple multi-task methods, such as VAR and MLR, that are not specifically designed for load forecasting, make the highest error in their predictions due to the simple linear combinations of previous load consumptions and entities. The multi-task method MTGP achieves better performance but is still outperformed by \mbox{Multi-APLF}, that achieves the best performance in all the entities of the datasets.

\begin{figure*}[h!]
    \centering
    \begin{subfigure}[b]{0.45\textwidth}
        \centering
        \psfrag{Load consumption [GW]}[b][][0.7]{Load consumption [GW]}
        \psfrag{Time [Hours]}[t][][0.7]{\raisebox{-5pt}{Time [Hours]}}
        \psfrag{data1abcdefghijklmn}[l][l][0.7]{Load}
        \psfrag{multi}[l][l][0.7]{\mbox{Multi-APLF}}
        \psfrag{mtgp}[l][l][0.7]{MTGP}
        \psfrag{varma}[l][l][0.7]{VAR}
        \psfrag{mlr}[l][l][0.7]{MLR}
        \psfrag{nhits}[l][l][0.7]{N-HITS}
        \psfrag{single}[l][l][0.7]{APLF}
        \psfrag{9}[][][0.7]{9\hspace{0.1cm}}
        \psfrag{10}[][][0.7]{10\hspace{0.15cm}}
        \psfrag{11}[][][0.7]{11\hspace{0.15cm}}
        \psfrag{12}[][][0.7]{12\hspace{0.15cm}}
        \psfrag{13}[][][0.7]{13\hspace{0.15cm}}
        \psfrag{14}[][][0.7]{14\hspace{0.15cm}}
        \psfrag{15}[][][0.7]{15\hspace{0.15cm}}
        \psfrag{16}[][][0.7]{16\hspace{0.15cm}}
        \psfrag{6}[][][0.5]{}
        \psfrag{120}[][][0.7]{\raisebox{-15pt}{12}}
        \psfrag{18}[][][0.5]{}
        \psfrag{24}[][][0.7]{\raisebox{-15pt}{24}}
        \psfrag{30}[][][0.5]{}
        \psfrag{36}[][][0.7]{\raisebox{-15pt}{36}}
        \psfrag{42}[][][0.5]{}
        \psfrag{48}[][][0.7]{\raisebox{-15pt}{48}}
        \includegraphics[width=\textwidth]{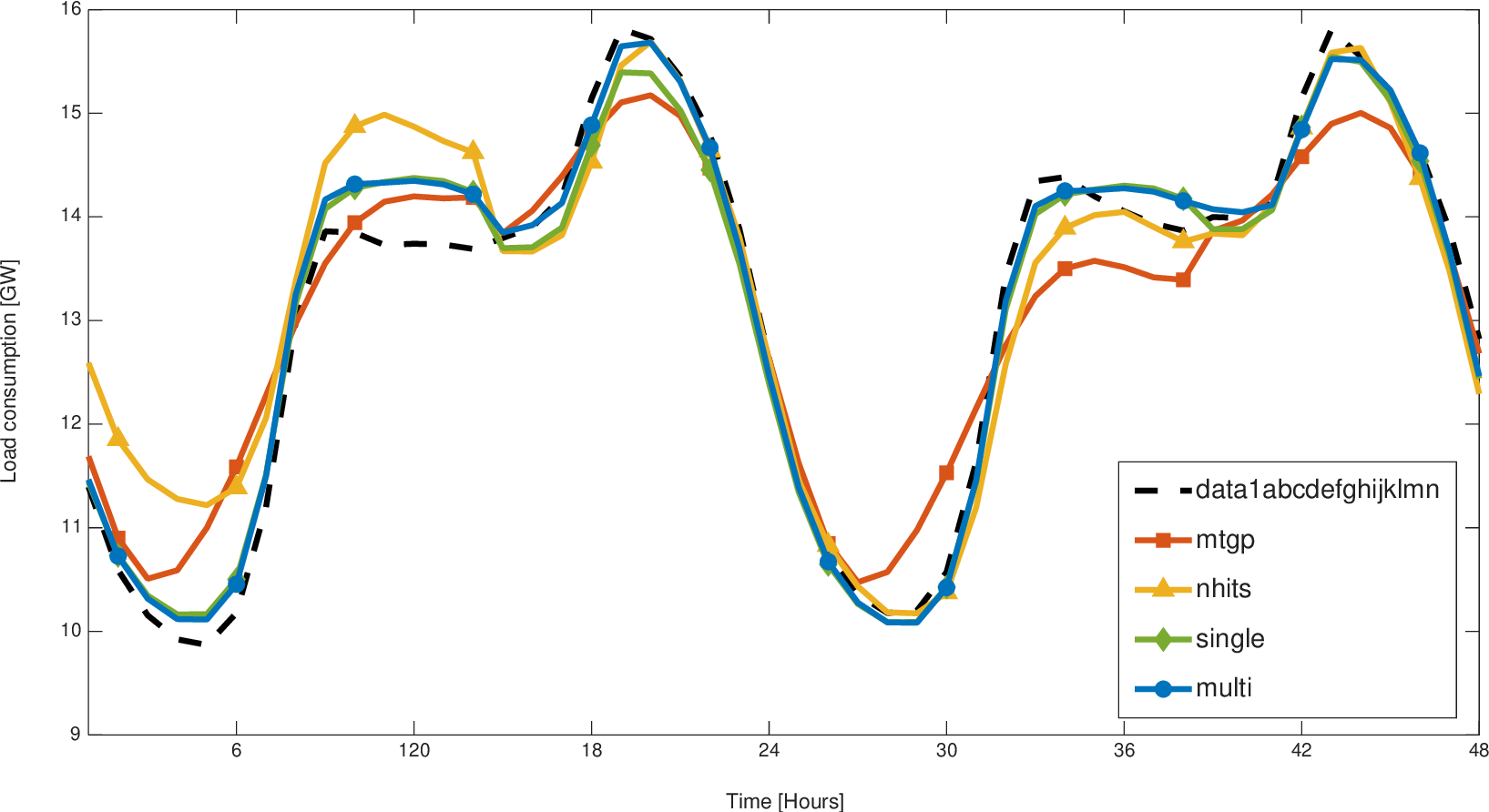}
        \caption{Example of a two-day prediction for GEFCom dataset}
        \label{fig:gefcom2017_loads}
    \end{subfigure}
    \hspace{0.5cm}
        \begin{subfigure}[b]{0.45\textwidth}
        \centering
        \psfrag{CDF}[b][][0.7]{CDF\raisebox{-5pt}{}}
        \psfrag{Error [GW]}[t][][0.7]{\raisebox{-5pt}{Error [GW]}}
        \psfrag{data1abcdefghijklmn}[l][l][0.7]{\mbox{Multi-APLF}}
        \psfrag{data3}[l][l][0.7]{MTGP}
        \psfrag{data4}[l][l][0.7]{VAR}
        \psfrag{data5}[l][l][0.7]{MLR}
        \psfrag{data6}[l][l][0.7]{N-HITS}
        \psfrag{data7}[l][l][0.7]{APLF}
        \psfrag{0.8}[][][0.7]{0.8 \hspace{0.15cm}}
        \psfrag{0.6}[][][0.7]{0.6 \hspace{0.15cm}}
        \psfrag{0.4}[][][0.7]{0.4 \hspace{0.15cm}}
        \psfrag{0.2}[][][0.7]{0.2 \hspace{0.15cm}}
        \psfrag{0.05}[][][0.7]{\raisebox{-15pt}{0.05} }
        \psfrag{0.10}[][][0.7]{\raisebox{-15pt}{0.10}}
        \psfrag{0.15}[][][0.7]{\raisebox{-15pt}{0.15}}
        \psfrag{0.20}[][][0.7]{\raisebox{-15pt}{0.20} }
        \psfrag{0.25}[][][0.7]{\raisebox{-15pt}{0.25} }
        \psfrag{0.30}[][][0.7]{\raisebox{-15pt}{0.30} }
        \psfrag{1}[][][0.7]{1}
        \psfrag{0}[][][0.5]{}
        \includegraphics[width=\textwidth]{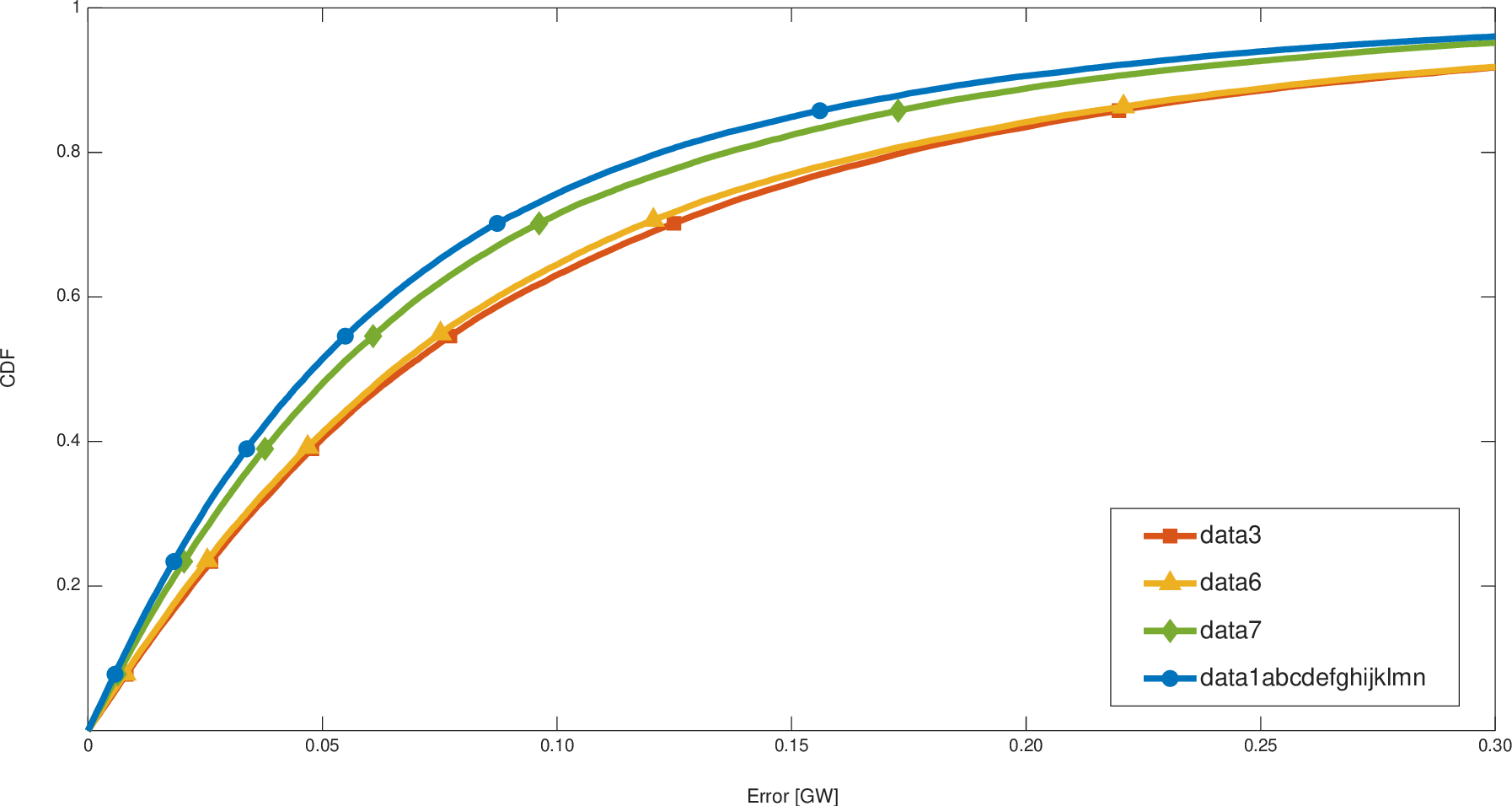}
        \caption{CDF for GEFCom dataset}
        \label{fig:gefcom2017_cdf}
    \end{subfigure}
    \caption{\mbox{Multi-APLF} method outperforms existing techniques and achieves more accurate predictions.}
    \label{fig:pred_performance}
\end{figure*}

\begin{figure*}[h!]
    \centering
    \begin{subfigure}[b]{0.45\textwidth}
        \centering
        \psfrag{C(q)}[][][0.7]{C(q)}
        \psfrag{Quantiles}[b][][0.7]{Quantiles}
        \psfrag{Multi-task}[l][l][0.7]{\mbox{Multi-APLF}}
        \psfrag{MTGP}[l][l][0.7]{MTGP}
        \psfrag{Single-task}[l][l][0.7]{APLF}
        \psfrag{abcdefghijklmnopqrstuvwxyzabcdefgh}[l][l][0.7]{Perfectly calibrated}
        \psfrag{0}[][][0.4]{}        
        \psfrag{1}[][][0.4]{}        
        \psfrag{2}[][][0.7]{0.2 \hspace{2mm}}
        \psfrag{3}[][][0.4]{}                
        \psfrag{4}[][][0.7]{0.4 \hspace{2mm}}
        \psfrag{5}[][][0.4]{}                
        \psfrag{6}[][][0.7]{0.6 \hspace{2mm}}
        \psfrag{7}[][][0.4]{}        
        \psfrag{8}[][][0.7]{0.8 \hspace{2mm}}
        \psfrag{9}[][][0.4]{}                
        \psfrag{10}[][][0.7]{1 \hspace{1mm}}
        \psfrag{0.1}[t][][0.5]{}
        \psfrag{0.2}[t][][0.7]{0.2}
        \psfrag{0.3}[t][][0.4]{}        
        \psfrag{0.4}[t][][0.7]{0.4}
        \psfrag{0.5}[t][][0.4]{}        
        \psfrag{0.6}[t][][0.7]{0.6}
        \psfrag{0.7}[t][][0.4]{}        
        \psfrag{0.8}[t][][0.7]{0.8}
        \psfrag{0.9}[t][][0.4]{}  
        \psfrag{C(q)}[][][0.7]{C(q)}
        \includegraphics[width=\textwidth]{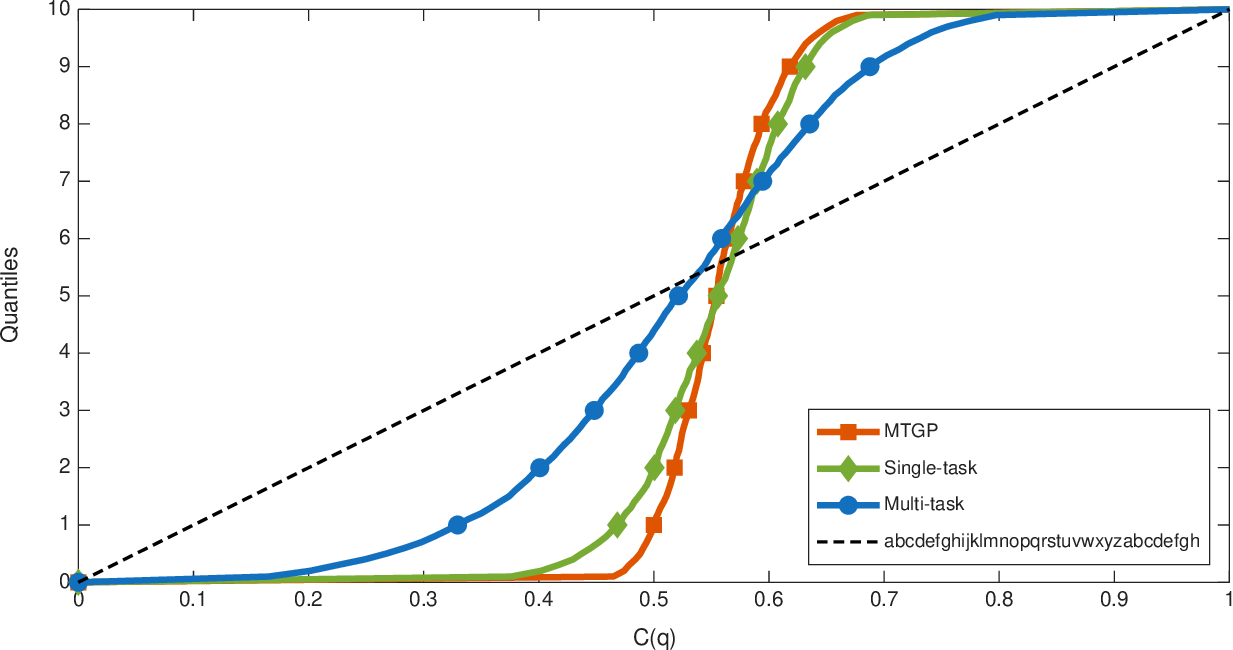}
        \caption{PJM dataset}
        \label{fig:PJM_cal}
    \end{subfigure}
    \hspace{0.5cm}
    \begin{subfigure}[b]{0.45\textwidth}
        \centering
        \psfrag{C(q)}[][][0.7]{C(q)}
        \psfrag{Quantiles}[b][][0.7]{Quantiles}
        \psfrag{Multi-task}[l][l][0.7]{\mbox{Multi-APLF}}
        \psfrag{MTGP}[l][l][0.7]{MTGP}
        \psfrag{Single-task}[l][l][0.7]{APLF}
        \psfrag{abcdefghijklmnopqrstuvwxyzabcdefgh}[l][l][0.7]{Perfectly calibrated}
        \psfrag{0}[][][0.4]{}        
        \psfrag{1}[][][0.4]{}        
        \psfrag{2}[][][0.7]{0.2 \hspace{2mm}}
        \psfrag{3}[][][0.4]{}                
        \psfrag{4}[][][0.7]{0.4 \hspace{2mm}}
        \psfrag{5}[][][0.4]{}                
        \psfrag{6}[][][0.7]{0.6 \hspace{2mm}}
        \psfrag{7}[][][0.4]{}        
        \psfrag{8}[][][0.7]{0.8 \hspace{2mm}}
        \psfrag{9}[][][0.4]{}                
        \psfrag{10}[][][0.7]{1 \hspace{1mm}}
        \psfrag{0.1}[t][][0.5]{}
        \psfrag{0.2}[t][][0.7]{0.2}
        \psfrag{0.3}[t][][0.4]{}        
        \psfrag{0.4}[t][][0.7]{0.4}
        \psfrag{0.5}[t][][0.4]{}        
        \psfrag{0.6}[t][][0.7]{0.6}
        \psfrag{0.7}[t][][0.4]{}        
        \psfrag{0.8}[t][][0.7]{0.8}
        \psfrag{0.9}[t][][0.4]{}   
        \includegraphics[width=\textwidth]{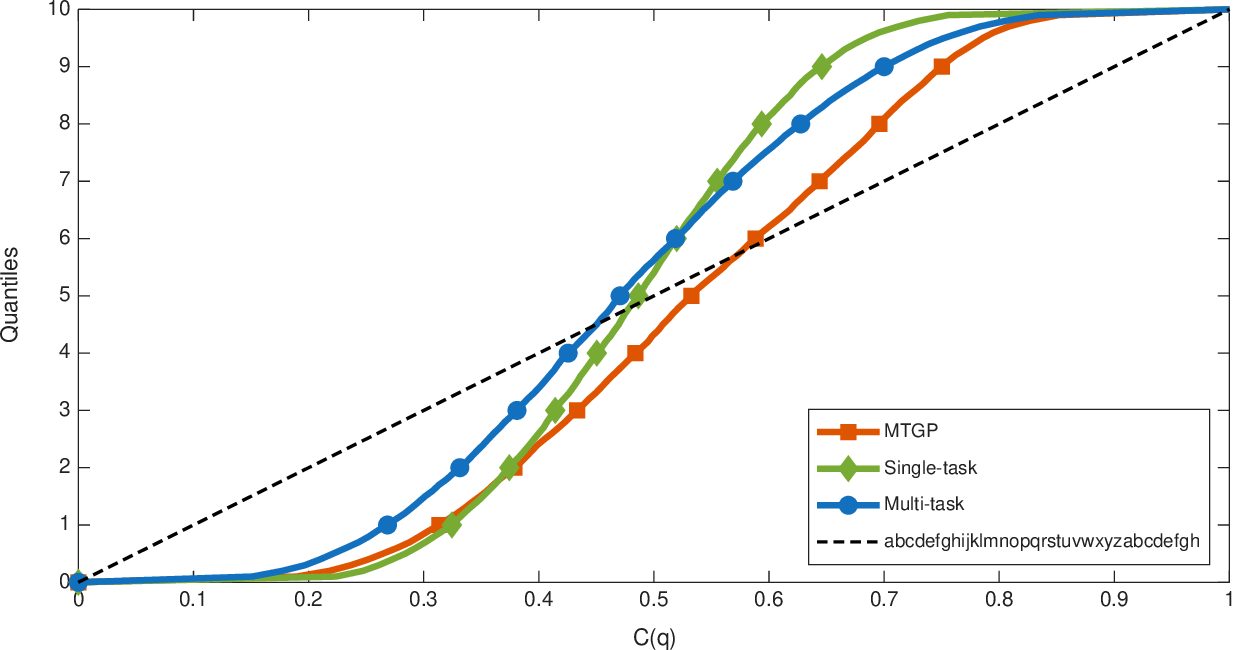}
        \caption{Australia dataset}
        \label{fig:Australia_cal}
    \end{subfigure}
    \caption{Calibration plots of different datasets evaluating the alignment between observed values and predicted quantiles.}
    % \vspace{-0.2cm}
    \label{fig:calibration}
\end{figure*}

The forecasting performance of the proposed method is illustrated in Figure~\ref{fig:pred_performance} for GEFCom dataset. Figure~\ref{fig:gefcom2017_loads} shows two days of load demand and load forecasts for the sum of loads of all the entities. Figure~\ref{fig:gefcom2017_cdf} presents the empirical cumulative distribution functions (CDFs) of the absolute prediction errors. These results further illustrate the performance improvement provided by the proposed method. In particular, Figure~\ref{fig:gefcom2017_cdf} shows that high errors occur with low probability for \mbox{Multi-APLF}. For instance, the error of the proposed method is smaller than 0.1 GW with a probability of 0.8, while most of the other methods reach errors up to 0.2 GW with such probability.

\begin{figure}
\centering
        \centering
        \psfrag{GEFCom}[l][l][0.7]{GEFCom}
        \psfrag{NewEnglandddddd}[l][l][0.7]{NewEngland}
        \psfrag{PJM}[l][l][0.7]{PJM}
        \psfrag{MAPE}[b][][0.7]{MAPE [\%]}
        \psfrag{Hours}[t][][0.7]{Delay [h]}
        \psfrag{0}[][][0.7]{0}        
        \psfrag{5}[][][0.7]{5}        
        \psfrag{10}[][][0.7]{10}
        \psfrag{15}[][][0.7]{15}
        \psfrag{20}[][][0.7]{20} 
        \psfrag{4}[][][0.7]{}   
        \psfrag{6}[][][0.7]{6}   
        \psfrag{7}[][][0.7]{7}   
        \psfrag{8}[][][0.7]{8}   
        \psfrag{9}[][][0.7]{9}   
        \includegraphics[width=0.5\textwidth]{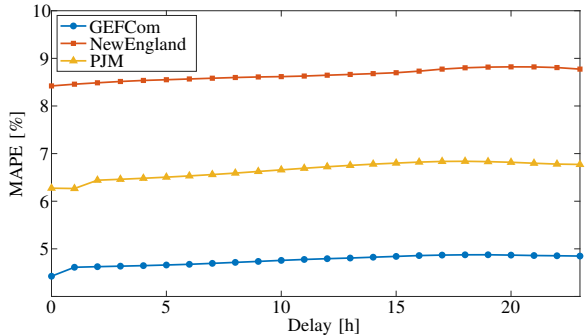}
    \caption{Forecasting error for three datasets varying input data delays.}
    \label{fig:hours_delay}
\end{figure}

\begin{figure}
\centering
        \centering
        \psfrag{CDF}[b][][0.7]{CDF}
        \psfrag{Error}[t][][0.7]{Error [MW]}
        \psfrag{0}[][][0.7]{0}        
        \psfrag{0.2}[][][0.7]{0.2\hspace{1mm}}        
        \psfrag{0.4}[][][0.7]{0.4\hspace{1mm}}
        \psfrag{0.6}[][][0.7]{0.6\hspace{1mm}}
        \psfrag{0.8}[][][0.7]{0.8\hspace{1mm}} 
        \psfrag{1}[][][0.7]{1\hspace{1mm}}   
        \psfrag{1.2}[][][0.7]{1.2}   
        \psfrag{1.6}[][][0.7]{1.6}   
        \psfrag{experiment1}[][][0.7]{\hspace{1mm} 2 entities}   
        \psfrag{experiment2}[][][0.7]{\hspace{1mm} 10 entities}   
        \psfrag{experiment3aa}[][][0.7]{40 entities}  
        \psfrag{5}[][][0.7]{5}  
        \psfrag{15}[][][0.7]{15}  
        \psfrag{25}[][][0.7]{25}  
        \psfrag{5}[][][0.7]{25}  
        \includegraphics[width=0.5\textwidth]{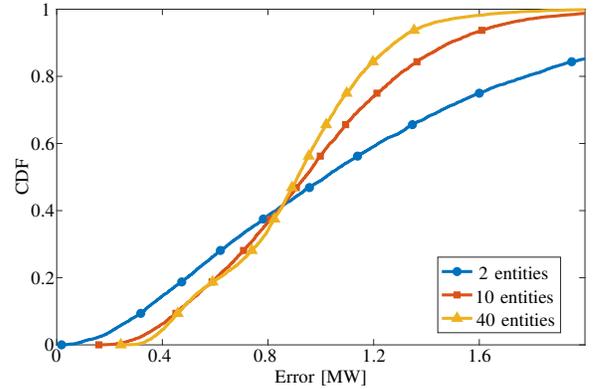}
    \caption{CDF of 3-month prediction errors varying the number of entities in the large-scale residential dataset.}
    \label{fig:increasing_entities}
\end{figure}

\begin{table}
\caption{Probabilistic performance of the proposed method and state-of-the-art techniques.}
\label{tab:prob_performance}
\setlength{\tabcolsep}{2.3pt} % Reduce column spacing
\renewcommand{\arraystretch}{1.2} % Adjust row height
\resizebox{\columnwidth}{!}{%
\large
\begin{tabular}{l|ccc|ccc|ccc}
\toprule
       & \multicolumn{3}{c|}{{APLF}} & \multicolumn{3}{c|}{{MTGP}} & \multicolumn{3}{c}{{\mbox{Multi-APLF}}} \\
      & \multicolumn{1}{c}{{CRPS}} &  \multicolumn{1}{c}{{Pinb.}} & \multicolumn{1}{c|}{{CE}} & \multicolumn{1}{c}{{CRPS}} &  \multicolumn{1}{c}{{Pinb.}} & \multicolumn{1}{c|}{{CE}} & \multicolumn{1}{c}{{CRPS}} &  \multicolumn{1}{c}{{Pinb.}} & \multicolumn{1}{c}{{CE}} \\ \hline
{GEFCom} & 0.29 & 0.14 & 0.18 & 0.37 & 0.18 & 0.19 & \textbf{0.23} & \textbf{0.11} &\textbf{0.12} \\
{NewEngl.} & 0.29 & 0.12 & 0.17 & 0.39 & 0.19 & 0.21 & \textbf{0.26} & \textbf{0.10} & \textbf{0.12} \\
{PJM} & 0.35 & 0.15 & 0.17 & 0.42 & 0.21 & 0.21 & \textbf{0.26} & \textbf{0.13} & \textbf{0.13}  \\
{Australia} & 0.11 & \textbf{0.05} & 0.14 & 0.23 & 0.11 & \textbf{0.11} & \textbf{0.10} & \textbf{0.05} & \textbf{0.11} \\
\bottomrule
\end{tabular}%
}
\end{table}

The second set of experiments evaluates the probabilistic performance of the proposed multi-task technique against the other two methods that provide probabilistic forecasts, APLF and MTGP. 
Table~\ref{tab:prob_performance} assesses the probabilistic performance of the proposed method in comparison with the APLF and MTGP techniques, using CRPS, pinball loss, and CE as evaluation metrics. Specifically, we evaluate the sum of the predictions of all individual entities. The proposed method achieves lower values for all the metrics across all datasets. This indicates that our method provides more accurate probabilistic predictions and better captures the uncertainty in load demand. Figure~\ref{fig:calibration} provides further detail on calibration. We can see that the proposed method achieves a closer alignment between predicted and observed probabilities than MTGP and APLF.

The third set of experiments further evaluates the practical applicability of the proposed method in more complex situations. Figure~\ref{fig:hours_delay} shows the results obtained in scenarios where the actual load demand is received with delays ranging from 1 to 23 hours, while Figure~\ref{fig:increasing_entities} shows the performance varying the number of forecasting entities. In particular, Figure~\ref{fig:hours_delay} illustrates the robustness of the proposed method to delayed load demand, demonstrating resilience to communication latency across `GEFCom', `NewEngland', and `PJM' datasets. For instance, in the New England dataset, the error increases only slightly (from 8.4 to 8.8), and remains consistently lower than that of offline methods  without delay as shown in Table~\ref{tab:experiments}.

Figure~\ref{fig:increasing_entities} shows the CDFs of the absolute prediction errors of the proposed method varying the number of forecasting entities using the `New South Wales' dataset that has data from 400 buildings. These numerical results show the CDFs of the absolute errors over $5$ random instantiations for each number of entities, where each entity is composed of $10$ buildings. As illustrated in Figure~\ref{fig:increasing_entities}, the errors decrease as the number of entities increases, demonstrating the robustness and scalability of the proposed multi-task forecasting approach in large-scale scenarios.

The proposed method achieves more accurate predictions and more reliable probabilistic forecasts than existing load forecasting techniques. The numerical results confirm that our method can improve load forecasting by sharing information among multiple entities, while adapting to dynamic changes in consumption patterns.

\section{Conclusion}
\label{sec:conclusion}
The paper presents adaptive multi-task learning methods
for probabilistic load forecasting (Multi-APLF). 
In particular, the proposed methods can adapt to changes in consumption patterns over time, learn the relationship among multiple entities, and assess load uncertainties. The parameters of the model are updated using a recursive algorithm, and the method provides probabilistic predictions with the most recent parameters. We describe the theoretical guarantees of online learning and probabilistic prediction steps of the method for multiple entities. We present an efficient implementation of the method with lower computational and memory complexity than existing multi-task approaches. In addition, the paper compares the proposed method with state-of-the-art techniques designed for both single-task and multi-task load forecasting. The experimental results show that the proposed method achieves higher performance and provides more reliable forecasts by effectively leveraging information shared across multiple entities.

\appendices
\section{Proof of theorem \ref{theorem1}}
\label{ap:proof_theorem1}
To find the parameters of the mean matrix $\bd{M}_i$ and the covariance matrix $\bd{\Sigma}_i$ that maximize \eqref{eq:weighted_loglikelihood} at each time $t_i$, we set the derivatives of the weighted log-likelihood to zero with respect to $\bd{M}_i$ and $\bd{\Sigma}_i$. The main goal is to obtain recursive equations that iteratively update these parameters that maximize \eqref{eq:weighted_loglikelihood} at each time step.
\begin{align*}
\frac{\partial L_i(\bd{M}_i, \bd{\Sigma}_i)  }{\partial \bd{M}_i} & = -  \sum_{j=1}^i  \lambda^{i-j}\bd{\Sigma}_i^{-1} (\bd{s}_{t_j} - \bd{M}_i\bd{u}_{t_j}) \bd{u}_{t_j}^\top = 0 \\
 & \hspace{-1.5cm} \iff \sum_{j=1}^i \lambda^{i-j} \bd{\Sigma}^{-1} \bd{s}_{t_j}\bd{u}_{t_j}^\top = \sum_{j=1}^i \lambda^{i-j} \bd{\Sigma}^{-1}\bd{M}_i \bd{u}_{t_j} \bd{u}_{t_j}^\top\\
& \hspace{-1.5cm} \iff \bd{M}_i = \left( \sum_{j=1}^i \lambda^{i-j} \bd{s}_{t_j}\bd{u}_{t_j}\right) \left( \sum_{j=1}^i \lambda^{i-j} \bd{u}_{t_j}\bd{u}_{t_j}^\top   \right)^{-1}.
\end{align*}
We denote 
$$\bd{P}_i = \left( \sum_{j=1}^i \lambda^{i-j} \bd{u}_{t_j}\bd{u}_{t_j}^\top   \right)^{-1}$$
then,
\begin{align}
    \bd{P}_i = &  \left(  \lambda\sum_{j=1}^{i-1} \lambda^{i-j-1}  \bd{u}_{t_j} \bd{u}_{t_j}^\top + \bd{u}_{t_i}\bd{u}_{t_i}^\top \right)^{-1} \nonumber\\
    = & \left( \lambda \bd{P}_{i-1}^{-1} + \bd{u}_{t_i}\bd{u}_{t_i}^\top \right)^{-1} \nonumber \\
    %= & (\lambda  \bd{P}_{i-1}^{-1})^{-1} - (\lambda  \bd{P}_{i-1}^{-1})^{-1}\bd{u}_i (\bd{I} + \bd{u}_i^\top \lambda^{-1} \bd{P}_{i-1} \bd{u}_i)^{-1}\bd{u}_i^\top \lambda^{-1} \bd{P}_{i-1}   \\
    %= &\frac{1}{\lambda} \bd{P}_{i-1} - \frac{1}{\lambda^2}\bd{P}_{i-1} \bd{u}_i(\bd{I} + \frac{1}{\lambda} \bd{u}_i^\top \bd{P}_{i-1} \bd{u}_i)^{-1}\bd{u}_i^\top  \bd{P}_{i-1} \\
    = & \frac{1}{\lambda} \left( \bd{P}_{i-1} - \bd{P}_{i-1} \bd{u}_{t_i} ( \lambda + \bd{u}_{t_i}^\top \bd{P}_{i-1} \bd{u}_{t_i})^{-1} \bd{u}_{t_i}^\top \bd{P}_{i-1} \right) \label{dem:P_1}\\
    = & \frac{1}{\lambda} \left( \bd{P}_{i-1} -  \bd{P}_{i-1} \bd{u}_{t_i} \bd{k}_i  \right) \label{dem:P_2}
\end{align} 
where \eqref{dem:P_1} and \eqref{dem:P_2} are obtained by using the matrix inversion lemma and
% \vspace{-0.2cm}
\begin{equation}
    \bd{k}_i =  ( \lambda + \bd{u}_{t_i}^\top \bd{P}_{i-1} \bd{u}_{t_i})^{-1} \bd{u}_{t_i}^\top \bd{P}_{i-1}. \label{dem:k}
    % \vspace{-0.2cm}
\end{equation}
We denote % \vspace{-0.2cm}
\begin{align*}
    \bd{Q}_i = & \sum_{j=1}^i \lambda^{i-j} \bd{s}_{t_j} \bd{u}_{t_j}^\top = \lambda \sum_{j=1}^{i-1} \lambda^{i-j-1} \bd{s}_{t_j} \bd{u}_{t_j}^\top + \bd{s}_{t_i} \bd{u}_{t_i}^\top \\
    = &  \lambda \bd{Q}_{i-1} + \bd{s}_{t_i}\bd{u}_{t_i}^\top.
\end{align*}
The matrix of means $\bd{M}_i$ can be defined in terms of $\bd{P}_i$ and $\bd{Q}_i$ as follows
\begin{align*}
    \bd{M}_i = & \bd{Q}_i\bd{P}_i\\
    = &(\lambda \bd{Q}_{i-1} + \bd{s}_{t_i} \bd{u}_{t_i}^\top) \frac{1}{\lambda} (\bd{P}_{i-1} -  \bd{P}_{i-1} \bd{u}_{t_i} \bd{k}_i)\\
    = & \bd{Q}_{i-1} \bd{P}_{i-1}  + \frac{1}{\lambda} \bd{s}_{t_i} \bd{u}_i^\top \bd{P}_{i-1} -  \bd{Q}_{i-1} \bd{P}_{i-1} \bd{u}_{t_i} \bd{k}_i\\
    & - \frac{1}{\lambda}   \bd{s}_{t_i} \bd{u}_{t_i}^\top  \bd{P}_{i-1} \bd{u}_{t_i} \bd{k}_i\\
    = &  \bd{M}_{i-1} + \frac{1}{\lambda} \bd{s}_{t_i}(\lambda + \bd{u}_{t_i}^\top \bd{P}_{i-1}  \bd{u}_{t_i}) \bd{k}_i  \\
    & -    \bd{M}_{i-1} \bd{u}_{t_i} \bd{k}_i- \bd{s}_{t_i} \bd{u}_{t_i}^\top \bd{P}_{i-1}  \bd{u}_{t_i}  \bd{k}_i \\
    %= & \bd{M}_{i-1}^\top  + \bd{k}_i (\bd{s}_i^\top + \frac{1}{\lambda} \bd{u}_i^\top \bd{P}_{i-1}  \bd{u}_i \bd{s}_i^\top- \bd{u}_i^\top\bd{M}_{i-1}^\top - \frac{1}{\lambda} \bd{u}_i^\top \bd{P}_{i-1}  \bd{u}_i \bd{s}_i^\top) \\
    = & \bd{M}_{i-1} + (\bd{s}_{t_i} - \bd{M}_{i-1} \bd{u}_{t_i})\bd{k}_i.
\end{align*}
The recursive update of the parameters of the mean matrix is obtained by replacing \eqref{dem:k} in the above equation so that we get
\begin{align}
    \bd{M}_i 
    = & \bd{M}_{i-1} + \frac{(\bd{s}_{t_i} - \bd{M}_{i-1} \bd{u}_{t_i})  \bd{u}_{t_i}^\top \bd{P}_{i-1}}{\lambda + \bd{u}_{t_i}^\top \bd{P}_{i-1}\bd{u}_{t_i}}.  \label{dem:M}
\end{align}
For the covariance matrix, we have that
\begin{align}
 \frac{ \partial L_i(\bd{M}_i,  \bd{\Sigma}_i)  }{\partial \bd{\Sigma}_i }   & = 0  \nonumber\\
  & \hspace{-2cm}  \iff \frac{1}{2} \sum_{j=1}^i \lambda^{j-i}   \bd{\Sigma}_i^{-1} (\bd{s}_{t_j} - \bd{M}_i\bd{u}_{t_j})(\bd{s}_{t_j} - \bd{M}_i\bd{u}_{t_j})^\top \bd{\Sigma}_i^{-1} \nonumber \\
 &\hspace{-1.2cm}  -\frac{1}{2} \sum_{j=1}^i \lambda^{j-i}  \bd{\Sigma}_i^{-1}  = 0   \nonumber\\  
& \hspace{-2cm}  \iff  \sum_{j=1}^i \lambda^{i-j} \bd{\Sigma}_i  =\sum_{j=1}^i \lambda^{j-i}  (\bd{s}_{t_j} - \bd{M}_i\bd{u}_{t_j})(\bd{s}_{t_j} - \bd{M}_i\bd{u}_{t_j})^\top.   \label{derivative_sigma}
\end{align}
If we denote $\bd{\gamma}_i=  \sum_{j=1}^i \lambda^{i-j}$, then
\begin{equation}
     \bd{\gamma}_i=  \lambda \sum_{j=1}^{i-1} \lambda^{i-j-1} + 1 = \lambda \gamma_{i-1} + 1. \label{dem:gamma_i}
\end{equation}
Substituting $\bd{\gamma}_i$ in \eqref{derivative_sigma}, we have that % \vspace{-0.2cm}
\begin{align}
\bd{\Sigma}_i \bd{\gamma}_i  = &  \sum_{j=1}^i \lambda^{i-j}  (\bd{s}_{t_j} - \bd{M}_i\bd{u}_{t_j})(\bd{s}_{t_j} - \bd{M}_i\bd{u}_{t_j})^\top  \nonumber \\
= & \lambda \sum_{j=1}^{i-1} \lambda^{i-j-1} (\bd{s}_{t_j}- \bd{M}_{i-1}\bd{u}_{t_j})(\bd{s}_{t_j}- \bd{M}_{i-1}\bd{u}_{t_j})^\top \nonumber \\
& + (\bd{s}_{t_i}- \bd{M}_{i}\bd{u}_{t_i})(\bd{s}_{t_i}- \bd{M}_{i}\bd{u}_{t_i})^\top \nonumber \\
= & \lambda \bd{\Sigma}_{i-1} \bd{\gamma}_{i-1} + (\bd{s}_{t_i}- \bd{M}_{i}\bd{u}_{t_i})(\bd{s}_{t_i}- \bd{M}_{i}\bd{u}_{t_i})^\top \nonumber\\
= &  \bd{\Sigma}_{i-1} (\bd{\gamma}_{i}-1) \nonumber \\
& + ( \bd{s}_{t_i}- \bd{M}_{i-1}\bd{u}_{t_i}) (1-\bd{k}_i^\top \bd{u}_{t_i})^2( \bd{s}_{t_i}- \bd{M}_{i-1}\bd{u}_{t_i})^\top.\label{sigma_derivative_1}
\end{align}

The last equality \eqref{sigma_derivative_1} is obtained by replacing $\bd{M}_i$ and $\gamma_{i-1}$ given by \eqref{dem:M} and \eqref{dem:gamma_i} respectively. Using \eqref{dem:k} in the following equation, we obtain that % \vspace{-0.1cm}
\begin{align}
    1-\bd{k}_i^\top \bd{u}_{t_i} = & 1-\frac{\bd{u}_{t_i}^\top \bd{P}_{i-1}^\top \bd{u}_{t_i}}{\lambda+\bd{u}_{t_i}^\top\bd{P}_{i-1}\bd{u}_{t_i}} = \frac{\lambda}{\lambda+\bd{u}_{t_i}^\top\bd{P}_{i-1}\bd{u}_{t_i}}. \label{k}
\end{align}
Then, replacing \eqref{k} in \eqref{sigma_derivative_1}, we have that
\begin{align*}
\bd{\Sigma}_i \bd{\gamma}_i  = &  \bd{\Sigma}_{i-1} (\bd{\gamma}_{i}-1) +   \frac{\lambda^2( \bd{s}_{t_i}- \bd{M}_{i-1}\bd{u}_{t_i})( \bd{s}_{t_i}- \bd{M}_{i-1}\bd{u}_{t_i})^\top}{(\lambda+\bd{u}_{t_i}^\top\bd{P}_{i-1}\bd{u}_{t_i})^{2}}. \nonumber % \vspace{-0.2cm}
\end{align*}
Hence, we obtain that the maximum likelihood estimator for the covariance matrix is updated recursively as follows % \vspace{-0.2cm}
\begin{align}
\bd{\Sigma}_i  = & \bd{\Sigma}_{i-1} \nonumber \\
    & - \frac{1}{\bd{\gamma}_i} \Big( \bd{\Sigma}_{i-1}  -  \frac{\lambda^2 ( \bd{s}_{t_i}- \bd{M}_{i-1}\bd{u}_{t_i})( \bd{s}_{t_i}- \bd{M}_{i-1}\bd{u}_{t_i})^\top }{(\lambda+\bd{u}_{t_i}^\top \bd{P}_{i-1}\bd{u}_{t_i})^2} \Big).  \nonumber 
\end{align}

\section{Proof of theorem \ref{theorem}}
\label{ap:proof_theorem2}
We prove the following lemma, which will be applied afterwards in the proof of Theorem \ref{theorem}. 
\begin{lemma}
\label{lemma}
Let $\set{N}(\bd{x}; \bd{a}, \bd{B})$ and $\mathcal{N}(\bd{y}; \bd{C}\bd{x}, \bd{D})$ be two Gaussian density functions with $\bd{x}, \bd{a} \in \mathbb{R}^n$, $\bd{y} \in \mathbb{R}^m$, $\bd{B} \in \mathbb{R}^{n \times n}$, \mbox{$\bd{C}\in \mathbb{R}^{n \times m}$}, and $\bd{D}\in \mathbb{R}^{m \times m}$. Then, we have % \vspace{-0.2cm}

\begin{align*}
 \set{N}(\bd{x}; & \bd{a}, \bd{B}) \ \mathcal{N}(\bd{y}; \bd{C}\bd{x}, \bd{D}) \\
 & = \set{N}(\bd{x}; \bd{E} \bd{e}, \bd{E}) \ \mathcal{N}(\bd{y}; \bd{C}\bd{a}, \bd{D} + \bd{C} \bd{B} \bd{C}^\top )
\end{align*}
with % \vspace{-0.3cm}
\begin{align}
\label{eq:E}
\bd{E} & = (\bd{B}^{-1} + \bd{C}^\top\bd{D}^{-1}\bd{C})^{-1}\\
\label{eq:e}
\bd{e} & = \bd{B}^{-1}\bd{a} + \bd{C}^\top \bd{D}^{-1} \bd{y}.
\end{align}

\begin{proof}
\begin{align}
    \set{N}(\bd{x}; & \bd{a}, \bd{B}) \ \mathcal{N}(\bd{y}; \bd{C}\bd{x}, \bd{D})  = \nonumber\\
    = & \dfrac{1}{(2\pi)^{n} |\bd{B}|^{1/2} |\bd{D}|^{1/2}} \exp  \Big{\lbrace} -\dfrac{1}{2} (\bd{x}-\bd{a})^\top \bd{B}^{-1}(\bd{x}-\bd{a})\nonumber\\
    & -\frac{1}{2} (\bd{y}-\bd{Cx})^\top \bd{D}^{-1}(\bd{y}-\bd{Cx})  \Big{\rbrace}.\label{eq:gaussian}
\end{align}

The exponent part of the above equation equals
\begin{align*}
      &  -\frac{1}{2} \big{(}\bd{x}^\top ( \bd{B}^{-1} + \bd{C}^\top\bd{D}^{-1}\bd{C}) \bd{x} - \bd{x}^\top(\bd{B}^{-1}\bd{a} + \bd{C}^\top \bd{D}^{-1} \bd{y})\\
      & -  (\bd{a}^\top\bd{B}^{-1} + \bd{y}^\top \bd{D}^{-1}\bd{C})\bd{x} + \bd{y}^\top \bd{D}^{-1}\bd{y} + \bd{a}^\top \bd{B}^{-1} \bd{a}\big{)}\\
   =&  -\frac{1}{2} \big{(}\bd{x}^\top \bd{E}^{-1} \bd{x} - \bd{x}^\top\bd{e} -  \bd{e}^\top\bd{x} + \bd{y}^\top \bd{D}^{-1}\bd{y} + \bd{a}^\top \bd{B}^{-1} \bd{a}\big{)}
\end{align*}
where the last equality is obtained using the definition of $\bd{E}$ and $\bd{e}$ in~\eqref{eq:E} and~\eqref{eq:e}, respectively. Multiplying by $\bd{E}^{-1}\bd{E}$ and using that $\bd{E}$ is a symmetric matrix, we have that 
\begin{align*}
      -\frac{1}{2}& \left((\bd{x}-\bd{a})^\top  \bd{B}^{-1}(\bd{x}-\bd{a})  + (\bd{y}-\bd{Cx})^\top \bd{D}^{-1}(\bd{y}-\bd{Cx})\right) \\
   =&  -\frac{1}{2} \big{(}\bd{x}^\top \bd{E}^{-1} \bd{x} - \bd{x}^\top\bd{E}^{-1}\bd{E}\bd{e}-  \bd{e}^\top\bd{E}\bd{E}^{-1}\bd{x}\\
   &+ \bd{y}^\top \bd{D}^{-1}\bd{y} + \bd{a}^\top \bd{B}^{-1} \bd{a}\big{)}\\
   =& -\frac{1}{2} \big{(}(\bd{x} - \bd{Ee})^\top \bd{E}^{-1}(\bd{x} - \bd{Ee}) - (\bd{Ee})^\top \bd{E}^{-1} \bd{Ee}\\
   &+ \bd{y}^\top \bd{D}^{-1}\bd{y} + \bd{a}^\top \bd{B}^{-1} \bd{a} \big{)}.
\end{align*}
Substituting the exponential part in~\eqref{eq:gaussian} by the above equality, we have that
\begin{align}
    \set{N}(\bd{x}; & \bd{a}, \bd{B}) \ \mathcal{N}(\bd{y}; \bd{C}\bd{x}, \bd{D}) \nonumber\\
    = & \dfrac{1}{(2\pi)^{n} |\bd{B}|^{1/2} |\bd{D}|^{1/2}} \exp \Big{\lbrace} -\frac{1}{2} (\bd{x} - \bd{Ee})^\top \bd{E}^{-1}(\bd{x} - \bd{Ee})\nonumber\\
    & -\frac{1}{2}\big{(} \bd{y}^\top \bd{D}^{-1}\bd{y} + \bd{a}^\top \bd{B}^{-1} \bd{a} -(\bd{Ee})^\top \bd{E}^{-1} \bd{Ee}\big{)} \Big{\rbrace}\nonumber\\
    = & \set{N}(\bd{x}; \bd{E} \bd{e}, \bd{E}) \dfrac{1}{(2\pi)^{n/2}|\bd{D} + \bd{C}^\top \bd{B} \bd{C}|^{1/2}}\nonumber\\
    \label{eq:gaussian_exp}
    &  \exp \Big{\lbrace}-\frac{1}{2}\big{(}\bd{y}^\top \bd{D}^{-1}\bd{y} + \bd{a}^\top \bd{B}^{-1} \bd{a} -(\bd{Ee})^\top \bd{E}^{-1} \bd{Ee}\big{)} \Big{\rbrace}
\end{align}
where the last equality is obtained by using that $$|\bd{E}| = \frac{1}{|\bd{B}^{-1} + \bd{C}\bd{D}^{-1}\bd{C}|} = \frac{|\bd{D}| |\bd{B}|}{|\bd{D} + \bd{C}^\top \bd{B} \bd{C}|}$$ due to the matrix determinant lemma. The exponential part of equation~\eqref{eq:gaussian_exp} equals
\begin{align*}
& - \frac{1}{2}\big{(} \bd{y}^\top (\bd{D}^{-1} - \bd{D}^{-1}\bd{C}\bd{E} \bd{C}^\top \bd{D}^{-1})\bd{y}-  \bd{a}^\top\bd{B}^{-1} \bd{E}  \bd{C}^\top \bd{D}^{-1} \bd{y} \\
& -  \bd{y}^\top \bd{D}^{-1}\bd{C} \bd{E} \bd{B}^{-1}\bd{a} + \bd{a}^\top \bd{B}^{-1} \bd{a}   - \bd{a}^\top\bd{B}^{-1} \bd{E} \bd{B}^{-1}\bd{a}\big{)}
\end{align*}
by using the definition of $\bd{e}$ in~\eqref{eq:e}. Then, the result is obtained since the above expression
equals
$$-\frac{1}{2}(\bd{y}-\bd{Ca})^\top \bd{F}^{-1}  (\bd{y}-\bd{Ca})$$
by using the matrix inversion lemma and   
\begin{align}
(\bd{D} + \bd{C} \bd{B} \bd{C}^\top) \bd{D}^{-1}\bd{C} \bd{E} \bd{B}^{-1}\bd{a} & = \bd{Ca}\\
\bd{a}^\top \bd{B}^{-1} \bd{a} - \bd{a}^\top\bd{B}^{-1} \bd{E} \bd{B}^{-1}\bd{a}   - \bd{a}^\top\bd{C}^\top \bd{F}^{-1} \bd{Ca} & =0.
\end{align}
\end{proof}
\end{lemma}

\begin{proof}[Proof of Theorem \ref{theorem}] 
We denote the sequences of observations and loads as \mbox{$\bd{r}_{t+1:t+i}= \{\bd{r}_{t+1},\bd{r}_{t+2},...,\bd{r}_{t+i} \}$} and \mbox{$\bd{s}_{t+i-1:t+i}= \{\bd{s}_{t+i-1},\bd{s}_{t+i-2},...,\bd{s}_{t+i} \}$} for any $i$.

We proceed by induction, for $i=1$ we have that   % \vspace{-0.1cm}
\begin{align}
    p&(  \bd{s}_{t+1}|\bd{s}_t, \bd{r}_{t+1}) \propto \   p(\bd{s}_{t+1},\bd{s}_t, \bd{r}_{t+1})  \nonumber  \\
    &= p(\bd{r}_{t+1} | \bd{s}_{t+1}, \bd{s}_t) p (\bd{s}_{t+1}| \bd{s}_{t}) p(\bd{s}_t) \nonumber \\
    & \propto p(\bd{r}_{t+1}| \bd{s}_{t+1}) p(\bd{s}_{t+1}|\bd{s}_t) \label{eq:p1} \\
    & \propto \set{N}(\bd{s}_{t+1}, \bd{M}_{r,c}\bd{u}_r, \boldsymbol{\Sigma}_{r,c}) \set{N}(\bd{s}_{t+1}, \bd{M}_{s,c}\bd{u}_s, \boldsymbol{\Sigma}_{s,c}) \label{eq:p2}
\end{align}
where proportionality relationships are due to the fact that $\bd{s}_t$ and $\bd{r}_{t+1}$ are known. Equation \eqref{eq:p1} is derived from the fact that the conditional distribution of $\bd{r}_{t+1}$ depends only on $\bd{s}_{t+1}$ by definition of the HMM, and equation \eqref{eq:p2} because we model these conditional distributions as Gaussian, as shown in equations \eqref{eq:model1_MT} and \eqref{eq:model2_MT}.

Using the previous Lemma \ref{lemma}, \eqref{eq:p2} leads to \eqref{eq:theorem} with $\widehat{\bd{s}}_{t+1}$ given by \eqref{eq:pred_mean} and $\widehat{\bd{E}}_{t+i}$ given by \eqref{eq:pred_error}, since $\widehat{\bd{s}}_t = \bd{s}_t$ and $\widehat{\bd{E}}_{t}= \bd{0}$. If the statement holds for $i-1$, then for $i$ we have that
% \vspace{-0.2cm}
\begin{align}
    p(&\bd{s}_{t+i}|\bd{s}_t, \bd{r}_{t+1:t+i}) \propto \   p(\bd{s}_{t+i},\bd{s}_t, \bd{r}_{t+1:t+i})  \nonumber  \\
    = \ & \int  p(\bd{s}_{t+i-1:t+i},\bd{s}_t, \bd{r}_{t+1:t+i})  d \bd{s}_{t+i-1} \label{eq:p3} \\ 
    = \ & \int  p(\bd{s}_{t+i-1:t+i},\bd{s}_t, \bd{r}_{t+1:t+i-1}) p(\bd{r}_{t+i} | \bd{s}_{t+i})  d\bd{s}_{t+i-1} \label{eq:p4}   \\
    = \ & p(\bd{r}_{t+i} | \bd{s}_{t+i}) \label{eq:p5}  \\
    & \cdot \int  p(\bd{s}_{t+i-1},\bd{s}_t, \bd{r}_{t+1:t+i-1}) p(\bd{s}_{t+i}|\bd{s}_{t+i-1})  d\bd{s}_{t+i-1} \nonumber  \\
    \propto \ &  p(\bd{r}_{t+i} | \bd{s}_{t+i}) \nonumber \\
    & \cdot \int  p(\bd{s}_{t+i-1}|\bd{s}_t, \bd{r}_{t+1:t+i-1}) p(\bd{s}_{t+i}|\bd{s}_{t+i-1})  d\bd{s}_{t+i-1} \nonumber  \\
    \propto \ &  \set{N}(\bd{s}_{t+i};\bd{M}_{r,c}\bd{u}_r, \bd{\Sigma}_{r,c}) \nonumber \\
      & \hspace{-0.5cm} \cdot \hspace{-0.15cm} \int \set{N}(\bd{s}_{t+i-1}; \hat{\bd{s}}_{t+i-1}, \hat{\bd{e}}_{t+i-1}) \set{N}(\bd{s}_{t+i}; \bd{M}_{s,c}\bd{u}_s, \bd{\Sigma}_{s,c})  d\bd{s}_{t+i-1}   \label{eq:p6} 
\end{align}
where proportionality relationships are due to the fact that $\bd{s}_t$ and $\bd{r}_{t+1:t+1}$ are known. Equation \eqref{eq:p3} is derived through marginalization, equation \eqref{eq:p4} and \eqref{eq:p5} are obtained by using the properties of HMMs, and equation \eqref{eq:p6} is obtained by applying the induction hypothesis and modeling of conditional distributions as Gaussian, as described in equations \eqref{eq:model1_MT} and \eqref{eq:model2_MT}. Subsequently, applying Lemma~\ref{lemma} and substituting $\bd{u}_s~=~[1, \ \bd{s}_{t+i-1}^\top]^\top $ into \eqref{eq:p6}, we obtain that % \vspace{-0.1cm}
\begin{align}
    p(\bd{s}_{t+i}|&\bd{s}_t, \bd{r}_{t+1:t+i} ) \nonumber\\
     & \hspace{-0.7cm} \propto \ \set{N}(\bd{s}_{t+i};\bd{M}_{r,c}\bd{u}_r, \bd{\Sigma}_{r,c}) \nonumber \\ 
    & \hspace{-0.4cm} \cdot \set{N}(\bd{s}_{t+i}; \bd{M}_s\widehat{\bd{u}}_{s}, \bd{\Sigma}_s + \bd{M}_s\bd{N}\widehat{\bd{E}}_{t+i-1}(\bd{M}_s\bd{N})^\top).  \label{eq:p8}
\end{align}
The result in the theorem is obtained by applying Lemma~\ref{lemma} in \eqref{eq:p8} again.
\end{proof}

\bibliographystyle{IEEEtran}
\bibliography{references.bib} 

\begin{IEEEbiography} 
[{\includegraphics[width=1in, height=1.25in, clip, keepaspectratio]{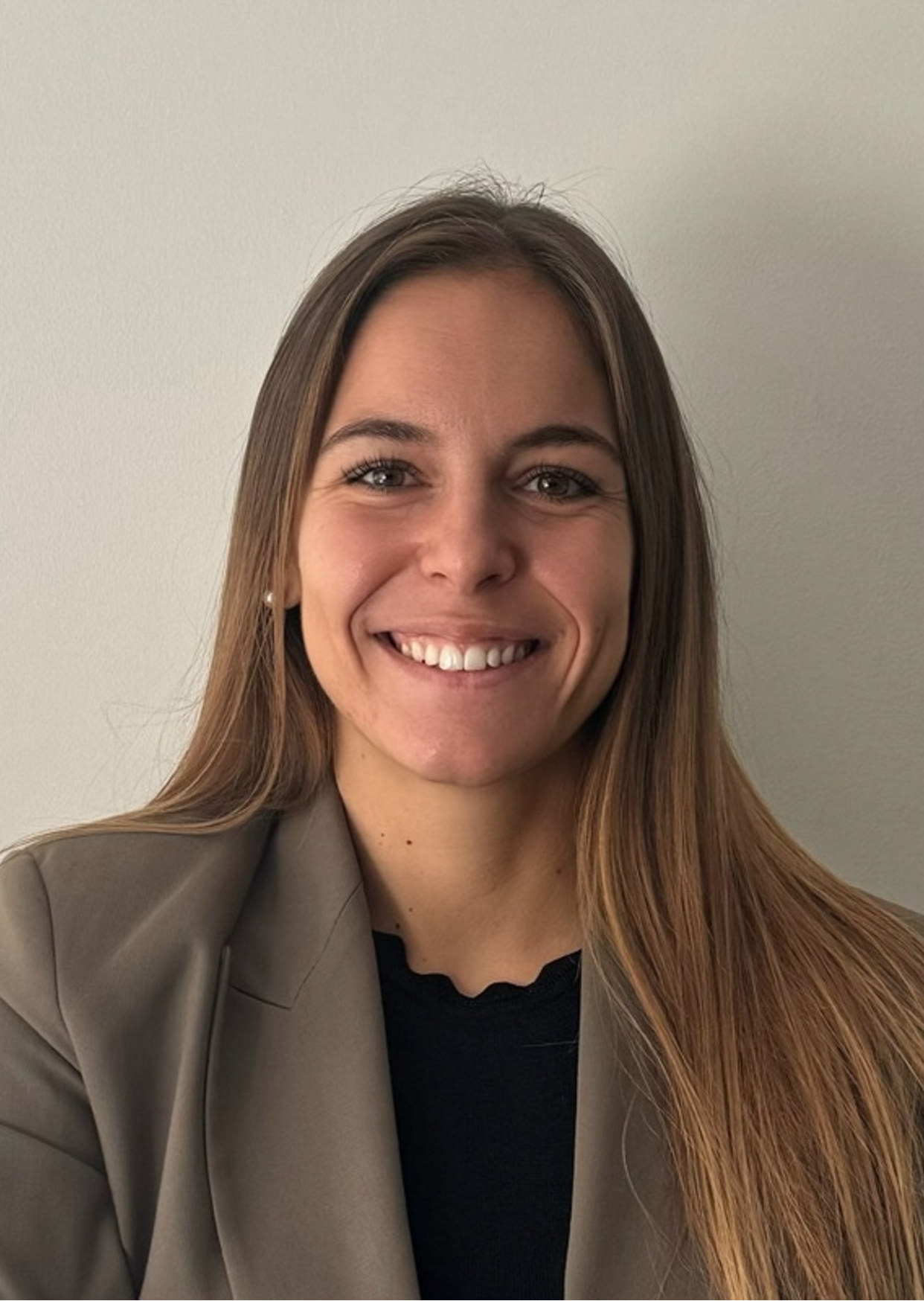}}]{Onintze Zaballa} obtained her Bachelor’s degree in Mathematics in 2017, and her Ph.D. degree in Computer Engineering in 2024 from the University of the Basque Country, Spain. She joined the Basque Center for Applied Mathematics (BCAM) in 2019, where she initially worked as a research technician. In 2020, Onintze started her Ph.D. under the guidance of Dr. Aritz Pérez and Prof. Jose A. Lozano and is currently working at Zenit Solar Tech. 

During her Ph.D., Onintze carried out a research stay at the Department of Applied Mathematics and Theoretical Physics, University of Cambridge. She co-organized the workshop ``The Mathematics of Machine Learning'' that took place at ETH Zurich in 2024 and she has served as a Technical Program Committee member for the IEEE Wireless Communications and Networking Conference in 2025. Her recent work on multi-task load forecasting has received the Best Paper Award in the 2024 IEEE Sustainable Power and Energy Conference (iSPEC).
\end{IEEEbiography}

\begin{IEEEbiography} 
[{\includegraphics[width=1in, height=1.25in, clip, keepaspectratio]{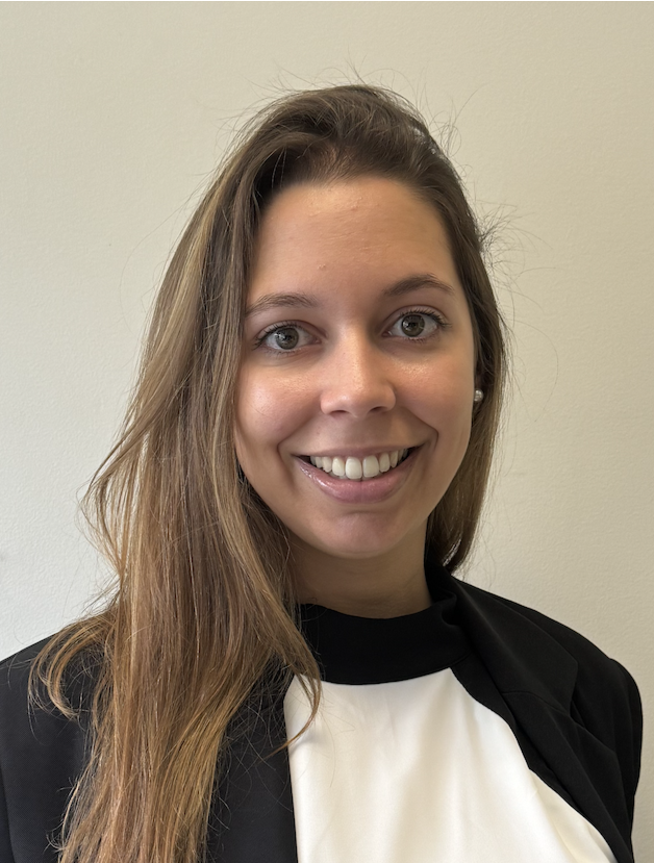}}]{Ver\'onica \'Alvarez} received the Bachelor’s degree in Mathematics from the University of Salamanca, Spain, in 2019 and Ph.D. degree in Computer Engineering from the University of the Basque Country, Spain, in 2023. 
She is currently Post-Doctoral Fellow at the Laboratory for Information and Decision Systems (LIDS) in the  Massachusetts Institute of Technology (MIT). Her main research interest is in the area of artificial intelligence, specifically in data science,
machine learning, and statistical signal processing. 

Ver\'onica co-organized the workshop ``The Mathematics of Machine Learning'' that took place at ETH Zurich in 2024. Verónica has received the Spanish Scientific Society of Computer Science (SCIE)-BBVA Foundation Research Award for Computer Science Young Researchers in 2025. Her PhD thesis has received the Best Thesis in the Field of Informatics and Artificial Intelligence awarded by the Spanish Society of Artificial Intelligence in 2024, and the Best Doctorate Thesis awarded by the University of the Basque Country in 2024, and an Honorable Mention of the Enrique Fuentes Quintana Award in 2023–2024 in the category of Engineering, Mathematics, Physics, Chemistry, and Architecture. In addition, she has received the 2022 SEIO-FBBVA Best Applied Contribution in Statistics Field for her work on adaptive probabilistic forecasting, and the Best Paper Award in the 2024 IEEE Sustainable Power and Energy Conference (iSPEC). 
\end{IEEEbiography}

\begin{IEEEbiography} 
[{\includegraphics[width=1.05in, height=1.4in, clip, keepaspectratio]{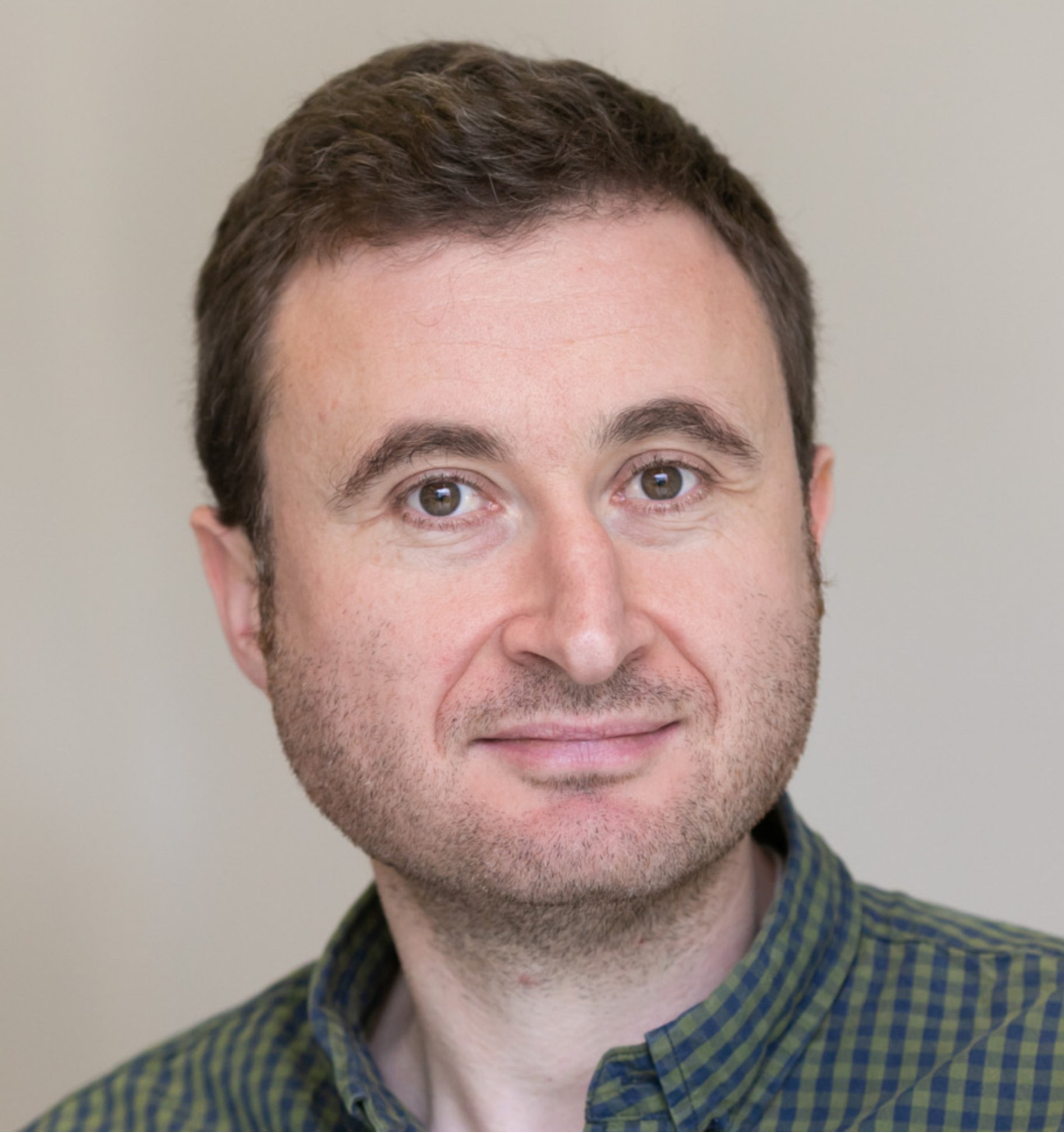}}]
{Santiago Mazuelas}
(Senior Member, IEEE) received the Ph.D. in Mathematics and Ph.D. in Telecommunications Engineering from the University of Valladolid, Spain, in 2009 and 2011, respectively.

Since 2017 he has been with the Basque Center for Applied Mathematics (BCAM) where he is currently Ikerbasque Research Professor. Prior to joining BCAM, he was a Staff Engineer at Qualcomm Corporate Research and Development from 2014 to 2017. He previously worked from 2009 to 2014 as Postdoctoral Fellow and Associate at the Laboratory for Information and Decision Systems (LIDS) at the Massachusetts Institute of Technology (MIT). 

Dr. Mazuelas was Area Editor for the IEEE Communication Letters from 2017 to 2022 and is currently Associate Editor-in-Chief for the IEEE 
Transactions on Mobile Computing, Associate Editor for the IEEE 
Transactions on Wireless Communications, and Action Editor for the Transactions on Machine Learning Research. He served as Technical Program Vice-chair at the 2021 IEEE Globecom as well as Symposium Co-chair at the 2014 IEEE Globecom, the 2015 IEEE ICC, and the 2020 IEEE ICC. He has received the Young Scientist Prize from the Union Radio-Scientifique Internationale (URSI) Symposium in 2007, and the Early Achievement Award from the IEEE ComSoc in 2018. His papers received the IEEE Communications Society Fred W. Ellersick Prize in 2012, the SEIO-BBVA Foundation Best Contribution in 2022 and 2025, and several Best Paper awards at international conferences.
\end{IEEEbiography} 

\end{document}